\documentclass[journal]{IEEEtran}
\IEEEoverridecommandlockouts 
\usepackage{amsmath,amsthm} 
\usepackage{amssymb}  
\usepackage{graphicx}
\usepackage{calc}

\usepackage{subcaption}
\usepackage{url}
\usepackage{booktabs}
\usepackage{enumitem,kantlipsum}
\usepackage{float}
\usepackage{textcomp}
\usepackage{graphics}
\usepackage{multirow}
\usepackage[dvipsnames]{xcolor}
\newtheorem{theorem}{Theorem}[section]

\newtheorem{corollary}[theorem]{Corollary}
\newtheorem{lemma}[theorem]{Lemma}
\theoremstyle{definition}
\newtheorem{definition}[theorem]{Definition}

\theoremstyle{remark}

\renewcommand{\baselinestretch}{1}
\usepackage{bm}
\newcommand{\kmax}{\kappa_{\text{max}}}
\newcommand{\ka}{\kappa}
\newcommand{\z}{\theta}

\newcommand{\be}{\begin{equation}}
\newcommand{\bs}{\begin{split}}
\newcommand{\es}{\end{split}}
\newcommand{\ee}{\end{equation}}
\newcommand{\kamax}{\kappa_{\text{max}}}

\newcommand{\ktop}{\kappa_{\text{top}}}

\newcommand{\sigmax}{\sigma_{\text{max}}}
\newcommand{\rhomax}{\rho_{\text{max}}}

\newcommand{\varD}{(\vect{p}_s, \kappa_{\text{top}}, \kappa_f, \Delta)}
\newcommand{\varmid}{(\vect{p}_s, \kappa_{\text{top}}, \kappa_f, \Delta=s_3)}

\newcommand{\varokmid}{(\vect{p}_s, \kappa_{\text{top}}, \kappa_f=0, \Delta=s_3)}

\newcommand{\J}{\mathcal{J}}
\newcommand{\vect}[1]{\boldsymbol{#1}}

\newcommand{\vmax}{v_{\text{max}}}

\usepackage{algorithm}
\usepackage[noend]{algpseudocode}

\title{\LARGE \bf
Trajectory Generation for Autonomous Vehicles Using Atropoids}

\title{\LARGE \bf Tunable Trajectory Planner Using $G3$ Curves}

\urldef{\smith}\url{stephen.smith@uwaterloo.ca}
\urldef{\botros}\url{alexander.botros@uwaterloo.ca}

\author{Alexander Botros and Stephen L.\ Smith
\thanks{This work is supported in part by the Natural Sciences and Engineering Research Council of Canada (NSERC)}
\thanks{The authors are with the Department of Electrical and Computer Engineering,
        University of Waterloo, 200 University Ave W, Waterloo ON, Canada, N2L 3G1
    (\botros, \smith)
        }
}

\begin{document}
\maketitle


\begin{abstract}
    Trajectory planning is commonly used as part of a local planner in autonomous driving. This paper considers the problem of planning a continuous-curvature-rate trajectory between fixed start and goal states that minimizes a tunable trade-off between passenger comfort and travel time. The problem is an instance of infinite dimensional optimization over two continuous functions: a path, and a velocity profile. We propose a simplification of this problem that facilitates the discretization of both functions.  This paper also proposes a method to quickly generate minimal-length paths between start and goal states based on a single tuning parameter: the second derivative of curvature. Furthermore, we discretize the set of velocity profiles along a given path into a selection of acceleration way-points along the path. Gradient-descent is then employed to minimize cost over feasible choices of the second derivative of curvature, and acceleration way-points, resulting in a method that repeatedly solves the path and velocity profiles in an iterative fashion. Numerical examples are provided to illustrate the benefits of the proposed methods.
\end{abstract}

\section{Introduction}
In motion planning for autonomous vehicles (or mobile robots), a trajectory generation technique is commonly used as part of a local planner~\cite{urmson2008autonomous}. The local planner produces many local trajectories satisfying constraints that reflect the vehicles dynamics, and discards those trajectories that collide with static or predicted dynamic obstacles. The best trajectory is selected according to a metric that captures the efficiency/length of the trajectory, and its smoothness or comfort~\cite{paden2016survey}. The trajectory is used as the reference for a tracking controller for a fixed amount of time, before the local planner produces a new reference. 

In this paper, we address the problem of generating a trajectory for an autonomous vehicle from a start state to a goal state. We focus on states of the form $(x,y,\z, \ka, v)$ where $(x,y)\in\mathbb{R}^2$ represents the planar location of the vehicle, $\z$ the heading, $\ka$ the curvature, and $v$ the velocity. In particular, we focus on computing trajectories that optimize a trade off between comfort and travel time~\cite{levinson2011towards, ziegler2014trajectory}. 

Comfort of a trajectory is related to the acceleration, rate of change of yaw, and jerk (i.e., the derivative of the acceleration) experienced by a particle (vehicle) sliding along it~\cite{ziegler2014trajectory}, \cite{huang2004fundamental}, \cite{articlejerk}. A common metric for comfort is the integral of the square of the jerk (IS jerk) over the trajectory~\cite{levinson2011towards}. When a vehicle is moving at constant speed, jerk is experienced lateral to the trajectory and is directly proportional to the derivative of curvature---called the \emph{sharpness}---of the vehicle trajectory. It is important to note that the smoothness and comfort of the resulting vehicle motion depends on both the reference trajectory and the tracking controller. 
However, by planning trajectories with low jerk, less effort is required from the controller to track the reference and ensure comfort~\cite{kyriakopoulos1994minimum}, particularly at high speeds~\cite{shin1992path}.

Often, trajectory planning is treated as a two part spatio-temporal problem~\cite{katrakazas2015real}. The first sub-problem deals with computing a \emph{path} from start to goal, while the second addresses how this path should be converted into a motion by computing a \emph{velocity profile}. Optimizing a travel time/comfort trade off over a set of paths and velocity profiles ensures the resulting trajectory's adherence to desirable properties like short travel time, and comfort. 

\begin{figure}[t]
    \centering
    \includegraphics[width=\linewidth]{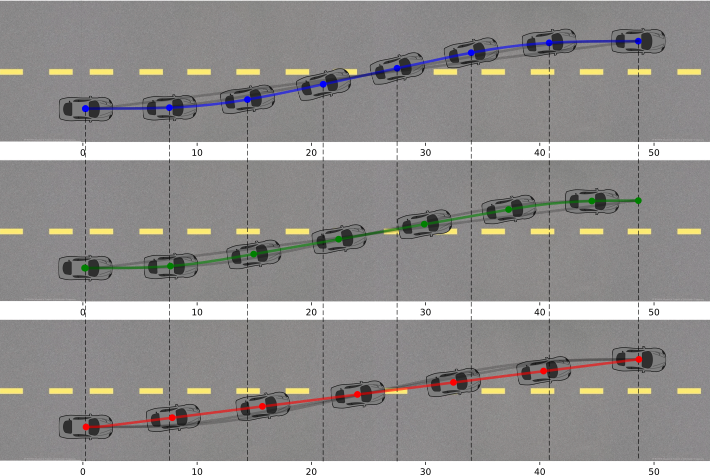}
    \caption{Example reference trajectories for a lane change maneuver given three minimization objectives: discomfort (top), time (bottom), and a mix between time and discomfort (mid). Initial and final speeds are fixed at 10 m/s, cars are drawn ever 0.75 seconds.}
    \label{PNWHL}
\end{figure}

The problem just described is an instance of infinite dimensional optimization in which a non-convex cost (discussed later) is minimized over two continuous functions: path and velocity profile. On one extreme, one could attempt to compute both functions simultaneously.  However, because of the nature of the problem, conventional techniques like convex optimization \cite{gb08}, or sequential quadratic programming, cannot be used without first modifying the problem.  On the other extreme, one could consider a simplified version of the problem by completely decoupling path and velocity: first minimizing cost over one function, then the other~\cite{kant1986toward}.  This latter approach does not consider the influence of the choice of path on the velocity profile (or vice-versa). For example, if a user strongly prefers short travel time to comfort, then an algorithm which first selects a cost-minimizing path, and then computes a velocity given that path would compute a trajectory with the shortest possible path, and highest feasible velocity.  However, it has been observed that such a trajectory is sub-optimal when attempting to minimize travel time.  This is because short paths typically feature higher values of curvature which require lower values of velocity to maintain safety.   

We propose in this paper a hybrid of these two approaches: we decouple the problem but repeatedly solve the path and velocity profiles in an iterative fashion.  In particular, we consider a simplification of the original optimization problem that limits the set of admissible paths to one in which individual elements can easily be distinguished from each other via a single parameter $\bar{\rho}\in\mathbb{R}_{>0}$.  This is done by considering only those paths whose second derivative of curvature is piece-wise constant, taking only values in $\{0, \pm \bar{\rho}\}$.  We also propose a modification of the techniques employed by \cite{biegler1984solution} to discretize the set of admissible velocity profiles into $n-1\geq 2$ way-points for longitudinal acceleration.  Thus, the two continuous functions, path and velocity profile, over which the our cost is minimized are replaced with $n$ constants.  We then repeatedly select paths by selecting values for $\bar{\rho}$, and compute the $n-1$ remaining parameters that minimize cost given the selected path.  We thus iteratively refine both path and velocity profile.  Our contributions are summarized in the next section.

\subsection{Contributions}
This paper contributes to the work of trajectory planning in the following ways:
\begin{itemize}
    \item Given a set of weights representing a trade off between comfort and travel time, we propose a method of simplifying the resulting infinite-dimensional non-convex optimization problem to one of finite dimension, allowing us to iteratively refine both the path and velocity. 
    \item To use the technique proposed here, we develop a method to compute paths between start and goal configurations whose second derivative of curvature is piece-wise constant taking values only in $\{0, \pm \bar{\rho}\}$ given a value of $\bar{\rho}\in\mathbb{R}_{>0}$.  
    \item Finally, we present a modification of the technique from  \cite{biegler1984solution} that allows us to compute a cost-minimizing velocity profile given a path.
\end{itemize}

\subsection{Related Work}
\subsubsection{Path Planning}
Path planning techniques can typically be divided into two categories: a path-fitting approach, and a curvature fitting approach. In the path fitting approach, a form of the path is fixed up to unknown parameters. These parameters are then computed numerically by solving optimization problems. Example path-forms include $G2-$splines~\cite{inproceedings2, inproceedings}, and Bezier curves~\cite{5233184, 5548085}. In the path fitting approach, smoothness and comfort of paths is typically enforced by minimizing the maximum curvature rate (the rate of change of curvature with respect to arc-length) over the trajectory. Minimizing the maximum curvature rate as opposed to simply bounding it may result in longer than necessary paths. 
\par
Unlike the path fitting approach, the curvature fitting approach assumes a form of the curvature profile over the path (as opposed to the path itself) up to unknown parameters. The unknown parameters are then computed by solving a two point boundary value problem between start and goal configurations. Early work in this approach includes Dubins' paths \cite{Dubins_1957}, which are the shortest paths between start and goal configurations if the curvature is bounded. 
\par
In~\cite{Kelly-2003-8714, 6225063}, the authors use a cubic polynomial curvature representation. However, the resulting path does not possess sufficient degrees of freedom to account for bounds on the curvature, resulting in potentially infeasible paths that violate the physical curvature constraints of the vehicle. 

Clothoid paths~\cite{1362698} address feasibility issues by placing constraints on both curvature and curvature rate over the path resulting in paths with piece-wise constant curvature rate. A \emph{clothoid}, or Euler Spiral, is a curve whose instantaneous curvature, $\ka$ is a linear function of the curve's arc-length, $s$. When the start and goal configurations are given as an $(x,y)$ pose, a heading $\theta$, and a curvature $\kappa$, several methods have been proposed to generate the shortest path from start to goal as a sequence of clothoid and straight line segments~\cite{shin1992path},~\cite{Gim2017},~\cite{8461201}. However, since the curvature of a clothoid varies linearly with arc length, the sharpness of the path (and thus the lateral jerk) is constant. Hence, a clothoid is a path that minimizes the maximum squared jerk rather than the IS jerk. 

Requiring that the curvature be twice differentiable, and placing bounds on curvature, curvature rate, and second derivative of curvature results in paths called $G3$ paths whose IS jerk (and not simply maximum squared jerk) can be minimized. Such paths are to clothoid ($G2$) paths, what clothoids were to Dubin's ($G1$) paths. Thus $G3$ paths are one step closer to paths with infinitely differentiable curvatures -- like spline and Bezier paths -- while minimizing arc-length in the presence of explicit restrictions on curvature, and curvature rate (like clothoid paths). In~\cite{oliveira2018trajectory}, the authors obtain $G3$ paths by concatenating cubic spline curves and straight lines. Though this approach results in $G3$ paths, the spline segments of the path, which are infinitely differentiable, may result in longer than necessary arc-lengths as is the case with regular spline paths. 

In~\cite{banzhaf2018g}, the authors introduce continuous curvature rate (CCR), and hybrid curvature rate (HCR) curves. These curves are $G3$ curves in which the derivative of curvature with respect to arc-length, $\sigma$, is a piece-wise linear function of the arc-length $s$, and the second derivative of curvature with respect to arc-length, $\rho$, is piece-wise constant. The authors of~\cite{banzhaf2018g} compute paths between start and goal states by combining CCR/HCR curves and straight lines. 

By placing bounds on $\sigma$ and $\rho$, the paths developed in \cite{banzhaf2018g} have the added benefit of bounding not only lateral/longitudinal jerk, but angular jerk (the rate of change of angular acceleration), as well as lateral/longitudinal snap (the time derivative of jerk). These benefits make HCR/CCR curves particularly attractive for use in the path planning phase of trajectory planning. The techniques developed in~\cite{banzhaf2018g} will be extended in this paper in four ways. First, the authors of~\cite{banzhaf2018g} focus on paths whose start and goal states have either 0 or maximum curvature. Though the symmetry resulting from this assumption greatly simplifies the mathematics of path planning, it limits the usefulness of the approach in producing a tracking trajectory. In particular, while tracking a reference trajectory, it may occur that the vehicle slips off of the planned path and requires a re-plan from a state whose curvature is not 0 or its maximum value. Therefore, the work herein relaxes the assumptions on the initial and final curvatures. Second, the authors of~\cite{banzhaf2018g} use search techniques to join HCR/CCR curves with straight lines to produce a path. In this paper, we present an algorithm to quickly determine such paths by reducing the path planning problem to a Dubin's-like sub-problem. Third, it is assumed in~\cite{banzhaf2018g}, that the maximum second derivative of curvature with respect to arc-length, $\rhomax$, along a path is known. Because $G3$ curves are categorized by a piece-wise constant functions $\rho(s)$, the value of $\rhomax$ dictates the slope -- and therefore the magnitude at any arc-length -- of the curvature rate.  Observe then, that $\rhomax$ should be functions of speed as, for example, a driver may turn the steering wheel very quickly in a parking lot, but not on a highway. In this paper, we compute $\rhomax$ by including it as a parameter denoted $\bar{\rho}$ in an optimization problem that seeks to minimize a trade off between travel time and passenger comfort.  Finally, it is assumed in~\cite{banzhaf2018g}, that the velocity of the vehicle is a positive or negative constant. While the work presented herein is limited to positive speeds, we develop a method by which a velocity profile can be computed.

\vspace{3mm}

\subsubsection{Velocity Planning}
Once a candidate path has been developed, a velocity profile to describe the motion of the vehicle along that path must be computed. This velocity profile should take into account a tradeoff between comfort and duration of travel. A common technique involves minimizing a cost function that is a weighted sum of undesirable features~\cite{levinson2011towards}. In~\cite{ziegler2014trajectory, ziegler2014making}, the authors integrate the weighted sum of the squared offset of the trajectory from the center of the road, the squared error in velocity from a desired profile, the squared acceleration, the squared jerk, and the squared yaw rate. A similar technique is employed in~\cite{li2015real} with the addition of a penalization on final arc-length. In the context of local trajectory planning, obstacles such as road boundaries, pedestrians, etc. are not considered. Therefore, in this work, we focus on developing a velocity profile and bound $\bar{\rho}$ that minimize a cost function similar to \cite{ziegler2014trajectory, ziegler2014making}, penalizing arc-length, acceleration, jerk, and yaw rate.

\section{Problem Statement}
We begin with a review of the optimization problem that motivated the development of HCR/CCR curves in~\cite{banzhaf2018g}. We will then augment the problem to account for comfort resulting in a new problem that is the focus of this paper. A note on notation: we use $(\cdot)'$ to denote differentiation with respect to arc-length $s$ along a path, while $\dot{(\cdot)}$ represents differentiation with respect to time $t$.
\subsection{Original Optimization Problem}
Let $\vect{p}=[x, y, \z, \ka, \sigma, v]^T$ denote a vector of states for a vehicle. Here, $(x,y)\in\mathbb{R}^2$ is the planar location of the  vehicle, $\z$ the heading, $\ka$ the curvature, $\sigma$ the curvature rate, and $v$ the velocity. In this paper, we refer to the states $x,y,\z, \ka, \sigma$ as \emph{path states}. We adopt the following common vehicle model for the path states, assuming forward motion~\cite{banzhaf2018g}
\begin{equation}
\label{dynamics}
    \begin{bmatrix}
    x'\\
    y'\\
    \theta'\\
    \kappa'\\
    \sigma'
    \end{bmatrix}=\begin{bmatrix}
    \cos(\theta)\\
    \sin(\theta)\\
    \kappa\\
    \sigma\\
    0
    \end{bmatrix}+\begin{bmatrix}
    0\\
    0\\
    0\\
    0\\
    \rho
    \end{bmatrix}.
\end{equation}
Note, $\sigma$ is the derivative of curvature with respect to arc-length (called the \emph{sharpness}, or \emph{curvature rate}), while $\rho$ represents the second derivative of curvature with respect to arc-length.  Observe that the model in~\eqref{dynamics} does not assume that $\rho$ is continuous. Thus model~\eqref{dynamics} can be used to describe $G3$ curves whose curvature rates are piece-wise linear functions of arc-length. The goal of path planning using $G3$ curves is to obtain curves of minimal final arc-length $s_f$, from start path states $(x_s, y_s, \z_s, \ka_s, \sigma_s)$ to a goal path states $(x_g, y_g, \z_g, \ka_g, \sigma_g)$, subject to boundary constraints:
\begin{equation}
\label{boundary}
    \begin{split}
        x(0)&=x_s, \  y(0)\ =y_s, \ \  \theta(0)\ =\theta_s, \ \kappa(0), \ =\kappa_s,\\
     x(s_f)&=x_g, \  y(s_f)=y_g, \  \theta(s_f)=\theta_g, \ \kappa(s_f)=\kappa_g,\\
     \sigma(0)&=\sigma(s_f)=0.
    \end{split}
\end{equation}
and physical constraints:
\begin{equation}
\label{bounds}
    \begin{split}
         \kappa(s)&\in[-\kmax, \kmax], \ \forall s\in[0, s_f], \\
     \sigma(s)&\in[-\sigmax, \sigmax], \ \forall s\in[0, s_f],\\
     \rho(s)&\in[-\rhomax, \rhomax], \ \forall s\in[0, s_f],
    \end{split}
\end{equation}
where $\kamax, \sigmax, \rhomax >0$ are known.  To summarize, if velocity is restricted to positive values, then the work of~\cite{banzhaf2018g} seeks to solve the following optimization problem (OP):
\begin{equation}
\label{OP}
\begin{split}
    \min_{\rho} \  &s_f\\
    s.t. \ \text{constraints} \  &\eqref{dynamics}, \eqref{boundary}, \eqref{bounds}
\end{split}
\end{equation}
Observe that at constant speed, minimizing arc-length as in \eqref{OP}, is equivalent to minimizing travel time.  
\subsection{Adding Comfort Constraints}
\label{CH2.2}
The constraints \eqref{bounds} involve the positive constants $\kamax$, $\sigmax$, $\rhomax$.  As in~\cite{1362698} and~\cite{banzhaf2018g} we assume that these values are known.  However, we assume that these values reflect maximum physical limits.  That is, $\kamax^{-1}$ is the minimum turning radius of the vehicle, $\sigmax$ is the maximum curvature rate given limitations on the vehicle's steering actuator, etc. Observe that by \eqref{dynamics}, any function $\rho$ uniquely defines path, denoted $P_{\rho}$, of path states from a fixed start. For the purposes of this paper, we assume known initial and final velocities $v_s, v_g$, (resp.) and zero initial and final acceleration. This ensures smooth concatenation of trajectories, and is common in trajectory generation \cite{4650924}. Under these assumption, a continuous velocity profile is uniquely determined by $v''$, denoted $\beta$.

The functions $\rho$ and $\beta$ will serve as the variables of an OP that balances travel time and discomfort. Any choice of $(\rho, \beta)$ uniquely defines a trajectory given by the tuple $(P_{\rho}, v_{\beta})$ where $P_{\rho}$ is a feasible solution to \eqref{OP}, and $v_{\beta}$ is the velocity profile associated with $\beta$. If $\rho$ and $\beta$ are parameterized by arc-length, we write $(P_{\rho}(s), v_{\beta}(s))$, whereas we write $(P_{\rho}(t), v_{\beta}(t))$ if $\rho$ and $\beta$ are parameterized by time.

Using a technique similar to~\cite{levinson2011towards,ziegler2014trajectory, ziegler2014making,li2015real}, we penalize a trade off between travel time and comfort:
\begin{equation}
    \label{CostFunc}
    \begin{split}
    C(P_{\rho}(t),&v_{\beta}(t))=\int_0^{t_f}(w_{a}C_{a}+w_{\mathcal{J}}C_{\mathcal{J}} + w_yC_y+w_t) dt,\\
    \text{where} \ C_{a} &= |a_N(t)|^2 + |a_T(t)|^2,\\
    C_{\mathcal{J}} &= \mathcal{J}_N(t)^2 + \mathcal{J}_T(t)^2,
    C_y = \left(\bigg|\frac{d \z(t)}{dt}\bigg| \right)^2.
    \end{split}
\end{equation}
Here, $C_a, C_{\mathcal{J}}, C_y$ represent the squared magnitude of the acceleration $\vect{a}$ (expressed using normal and tangential components $a_N, a_T$), the squared magnitude of jerk (expressed using normal and tangential components $\mathcal{J}_N, \mathcal{J}_T$), and the squared magnitude of yaw rate for the trajectory $(P_{\rho}(t),v_{\beta}(t))$. These costs are weighted with constants $w_a, w_{\mathcal{J}}, w_y$ representing the relative importance of each feature to a user. We refer to the terms $\int_{0}^{t_f}w_mC_m, m\in\{a, \J, y, t\}$ as the integral squared (IS) acceleration, IS jerk, IS yaw, and time cost, respectively.  

As will be discussed later, we employ a variant of sequential quadratic programming (SQP) to optimize the velocity profile $v$, for a given path. Therefore, we require that the limits of integration in the definition of the cost be fixed~\cite{biegler1984solution}. Noting that the final arc-length of a given path $P_{\rho}(t)$ is fixed, we re-parameterize the cost \eqref{CostFunc}  in terms of arc-length $s$. Letting $\vect{n}, \vect{\tau}$ represent the unit normal and unit tangent vectors, respectively, we observe that the acceleration vector $\vect{a}$, and the jerk vector $\vect{\J}$ are given by
\begin{equation}
\label{AJ}
\begin{split}
     \vect{a} &=a_N\vect{n} +  a_T\vect{\tau} =|\ka|v^2\vect{n} + a_T\vect{\tau} \\
     \vect{\J}&= \dot{\vect{a}} = (3va_T|\ka| + v^3\bar{\sigma})\vect{n} + (\dot{a}_T - \ka^2v^3)\vect{\tau},
\end{split}
\end{equation}
where $\bar{\sigma}=(|\ka|)'$. The expression for $\vect{\J}$ was obtained by differentiating $\vect{a}$ with respect to time (using the Frenet-Serret formula to integrate the normal and tangent vectors, see \cite{schot1978jerk}) and observing $\dot{m}=m' (ds/dt)=m' v$ for any path state $m$. Finally, letting $\alpha = v', \beta = \alpha ', b=\dot{a}_T$, we observe
\begin{equation}
    \label{alphabeta}
    a_T = \alpha v, \ \dot{a}_T=b = v(\beta v + \alpha^2).
\end{equation}
Combining \eqref{CostFunc}, \eqref{AJ}, and \eqref{alphabeta}, and integrating with respect to $s$ instead of $t$ yields:
\begin{equation}
\label{NEWCOST}
    \begin{split}
C(P_{\rho}(s),v_{\beta}(s))&=\int_0^{s_f}w_a\tilde{C}_a+w_{\J}\tilde{C}_{\J} + w_y\tilde{C}_y+w_t\tilde{C}_t ds,\\
        \tilde{C}_a &= v^3|\ka|^2+\alpha^2v,\\
        \tilde{C}_{\J} &= v^3(3\alpha|\ka| + \bar{\sigma} v)^2 + v(\beta v + \alpha^2 -\ka^2v^2)^2,\\
        \tilde{C}_y &= |\ka|^2 v,\\
        \tilde{C}_t &= v^{-1}.
    \end{split}
\end{equation}
We now use this cost together with the constraints developed earlier to state the new OP that is the focus of this paper:
\begin{equation}
\label{NEWOP}
\begin{split}
    \min_{\rho(s), \beta(s)} \  &C(P_{\rho}(s),v_{\beta}(s))\\
    s.t. \ &\text{constraints }  \eqref{dynamics},  \eqref{boundary}, \eqref{bounds}, \eqref{alphabeta}\\
    & \begin{bmatrix}
    v'_{\beta}(s) & \alpha'(s)
    \end{bmatrix}^T=\begin{bmatrix}
    \alpha(s) & \beta(s)
    \end{bmatrix}^T\\
     &v_{\beta}(0)=v_s, \ v_{\beta}(s_f)=v_f, \ \alpha(0)=\alpha(s_f)=0,\\
     &v_{\beta}(s)\in(0, \vmax], a(s)\in[-a_{\text{max}}, a_{\text{max}}],\\
     &b(s)\in[-b_{\text{max}}, b_{\text{max}}], \ \forall s\in[0, s_f].
\end{split}
\end{equation}
Let $R\times B$ be the set of all $(\rho(s), \beta(s))$ whose associated trajectories are feasible solutions to \eqref{NEWOP}. In the next section, we describe our solution approach in detail.

\section{Approach}
The optimization problem in \eqref{NEWOP} is an instance of infinite dimensional, non-convex optimization. The non-convexity of \eqref{NEWOP} is due to the cost $C_{\J}$ (and $\tilde{C}_{\J}$). Furthermore, the non-holonomic constraints \eqref{dynamics} require the evaluation cubic Fresnel integrals for which there is no closed form solution. For these reasons, we propose an approach to that simplifies \eqref{NEWOP}. 
In this section we present the high-level idea behind our technique, beginning with a motivating Theorem. 

\begin{theorem}[$G3$ Paths as Optimal]
\label{G3asOpt}
If $(\rho^*, \beta^*)$ is a solution to the OP \eqref{NEWOP}, then $\rho^*$ is piece-wise constant. That is, the optimal path $P_{\rho^*}$ is a $G3$ path. 
\end{theorem}
The result of this Theorem follows directly from the observation that $\rho$ does not appear explicitly in the cost of \eqref{NEWOP}, and appears linearly in the constraints. Thus the Hamiltonian is linear in $\rho$, and the result follows from Pontryagin's Minimum Principle~\cite[Chapter~12]{kamien2012dynamic}.

Theorem \ref{G3asOpt} implies that a path that solves \eqref{NEWOP} is $G3$.  However, there are many possible $G3$ paths connecting start and goal path states.  Therefore, the first part of our technique is to simplify the OP \eqref{NEWOP} by considering only those $G3$ curves $P_{\hat{\rho}}$ such that there exists a constant $\bar{\rho}\leq \rhomax$ for which $P_{\hat{\rho}}$ solves \eqref{OP} when $\rhomax$ is replaced with $\bar{\rho}$. That is, we approximate a solution to \eqref{NEWOP} by solving
\be
\label{OPAPPROX}
\begin{split}
\min_{\bar{\rho}\leq \rhomax} \left( \min_{\beta(s)\in B} C(P_{\hat{\rho}}(s), v_{\beta}(s))\right),\\
s.t., \hat{\rho}(s)\in R, \ \text{solves \eqref{OP} for } \rhomax=\bar{\rho}.
\end{split}
\ee
In words, \eqref{OPAPPROX} approximates the OP in \eqref{NEWOP} by replacing the continuous function $\rho(s)$ with a constant $\bar{\rho}$. The path associated with $\bar{\rho}$ is a shortest $G3$ path $P_{\hat{\rho}}(s)$ such that $\hat{\rho}(s)$ is bounded in magnitude by $\bar{\rho}$.  From the results in \cite{banzhaf2018g}, we observe that $\hat{\rho}$ is a function taking values only in $\{\pm \bar{\rho}, 0\}$. This approach has two major advantages: first, we have replaced the continuous function $\rho(s)$ in \eqref{NEWOP} with constant $\bar{\rho}$ while still maintaining the $G3$-form of an optimal path.  Second, by tuning $\bar{\rho}$, we can still produce $G3$ curves with both sharp (typically favored by users who value short travel times) and gradual curvature functions (favored by users who value comfortable trajectories).  

The second part of our technique involves simplifying \eqref{OPAPPROX} yet further by replacing the continuous function $\beta(s)$ with a sequence of way-points for longitudinal acceleration. Similar to \cite{biegler1984solution}, we discretize the arc-length along the path $P_{\hat{\rho}}$ into $n-1\in\mathbb{N}_{\geq 2}$ fixed points $\{s_2,...,s_n\}$. We then replace the continuous function $\beta(s)$ with $n - 1$ constants $\alpha(s_i), i=2,\dots,n$ representing longitudinal acceleration at arc-length $s_i$. These acceleration way-points can then be used to produce a continuous function $\beta(s)$ of degree $n-1$ whose derivative $\alpha(s)$ takes values $\alpha(s_i)$ at arc-length $s_i$. In this work, we use $n=12$ which works well in practice.

Thus far, our approach has simplified \eqref{NEWOP} by replacing the two continuous functions $\rho, \beta$ with $n$ constants $\bar{\rho}, \alpha(s_i), i=2,\dots,n$.  The final stage of our technique is to compute the cost-minimizing values of these constants using gradient-decent to iteratively refine values of $\bar{\rho}$, and SQP to select $\alpha(s_i), i=2,\dots,n$ for each choice of $\bar{\rho}$.  In order to do this, we require two sub-techniques:  The first solves \eqref{OP} for any given values $\sigmax, \kamax,$ and $\rhomax=\bar{\rho}$ (Section IV). The second computes the optimal way-points $\alpha(s_i)$ for any fixed path (Section V).

\section{Computing $G3$ Paths}
\label{PathSec}
In this section, we describe how to compute $G3$ paths that solve \eqref{OP} for known values of $\rhomax$.  The function $\rho(s)$ that solves \eqref{OP} is piece-wise constant, taking values only in $\{0, \pm \rhomax\}$ (see \cite{banzhaf2018g}), a direct result of Pontryagin's Minimum Principle~\cite[Chapter~12]{kamien2012dynamic}. We begin with an investigation of single $G3$ curves, following closely the work presented in \cite{banzhaf2018g}. We then present a technique to connect $G3$ curves to form $G3$ paths.

\subsection{Single $G3$ Curves}
This section uses the definitions and notation presented in~\cite{banzhaf2018g} to outline the general form of a $G3$ curve. An example $G3$ curve performing a left-hand turn is presented in Figure~\ref{HCRcurve}. The maximum curvature of this maneuver is $\ktop$ where $|\ktop|\leq \kamax$. The curve begins at a path state $\vect{p}_s=(x_s, y_s, \z_s, \ka_s, \sigma_s=0)$ at an arc-length of $s=0$ along the path. From $s=0$ to $s=s_1$, the second derivative of curvature $\rho$, is set to its maximum value $\rhomax$. Thus the curvature rate $\sigma$ is given by the function $\sigma(s)=\rhomax\cdot s$. The functions $\ka(s), \z(s)$ are therefore quadratic and cubic (resp.), and can be calculated by~\eqref{dynamics}. 

At $s=s_1$, the curvature rate has reached its maximum allowable value $\sigmax$, and can go no higher, thus the function $\rho$ is set to 0, and $\sigma(s)=\sigmax$. As a result, for $s_1\leq s\leq s_2$, the function $\ka(s), \z(s)$ are linear and quadratic (resp.). At $s=s_2$, the value of $\rho$ is set to $-\rhomax$, and the curvature rate decreases from $\sigmax$, allowing the curvature to reach a value $\ktop$ at $s=s_3$. At $s=s_3$, the path state is given by $\vect{p}(s_3)$. From $s=s_3$ to $s=\Delta$, the curvature of the curve is held constant at $\ktop$. Thus the curve remains at a constant distance $\ktop^{-1}$ from its center of curvature $\vect{x}_c$ for all $s_3\leq s\leq\Delta$. The circle $\Omega_{\text{I}}$ is the circle centered at $\vect{x}_c$ with radius $\ktop^{-1}$. 
 \begin{figure}[t]
    \centering
 \includegraphics[width =0.9 \linewidth]{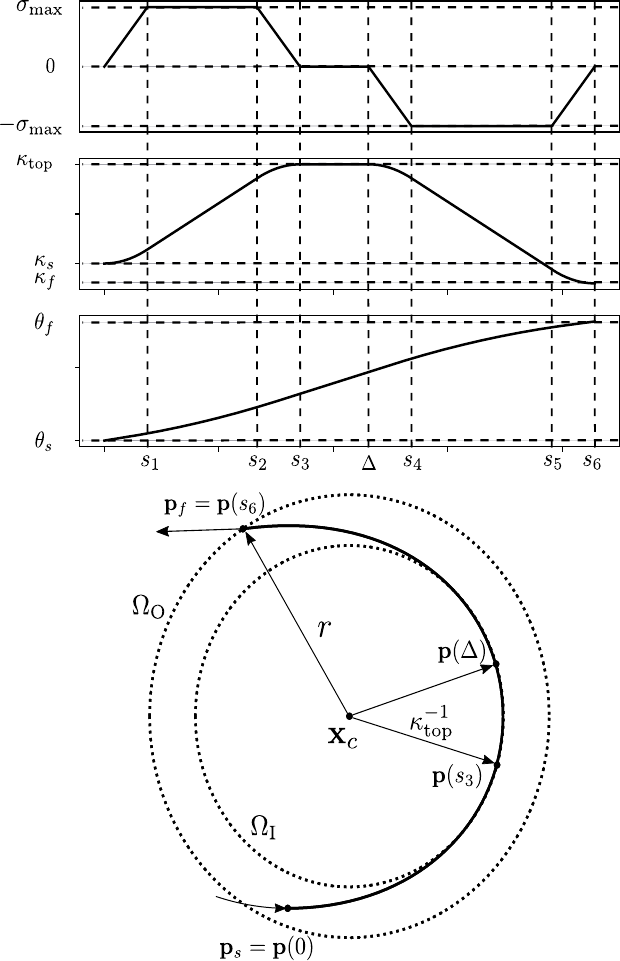}
    \caption{A basic $G3$ curve. \textbf{Top}: The functions $\sigma(s)$ (top), $\ka(s)$ (mid), and $\z(s)$ (bottom). \textbf{Bottom}: The resulting curve in the $x,y$ plane from start configuration $\vect{p}_s$ to final configuration $\vect{p}_f$. Image also appears in \cite{banzhaf2018g}}
    \label{HCRcurve}
\end{figure}  
\label{CurvePrelims}
At $s=\Delta$, the value of $\rho$ is set to $-\rhomax$, and the curvature rate follows the linear function $\sigma(s)=-\rhomax(s-\Delta)$ until $s=s_4$ when $\sigma(s)=-\sigmax$. Upon reaching its minimum allowable value, $\sigma(s)$ remains constant at $-\sigmax$ for $s_4\leq s\leq s_5$. At $s=s_5$, the value of $\rho$ is set to $\rhomax$, allowing the curvature to descend to its final value $\kappa_f$ at $s=s_6$. When $s=s_6$, the path state is given by $\vect{p}(s_6)$ which lies a distance of $r$ away from $\vect{x}_c$. The circle $\Omega_{\text{O}}$ is centered at $\vect{x}_c$ with radius $r$. We use a subscript $f$ to represent a path state at the end of a $G3$ curve. This is to differentiate these states from ones with a subscript $g$ which represents the goal path states at the end of a $G3$ path (the concatenation of one or more curves with straight lines).

To perform a right-hand turn, a similar analysis to that above is employed. The resulting function $\sigma(s)$ will appear as a mirror images about the horizontal axis to that presented in Figure~\ref{HCRcurve}. We now illustrate how the switching arc-lengths $s_i, i=1,\dots,6$ of the function $\rho(s)$ may be calculated. While similar, the switching arc-lengths presented in~\cite{banzhaf2018g} assume that $\ka_s, \ka_f\in\{0, \pm \kamax\}$. Because this is not an assumption that we make, it is necessary to re-derive these values.

Given the parameters $\rhomax,\sigmax,\ktop, \Delta$, the values of $s_i, \ i=1,\dots,6$ such that $\ka(0)=\ka_s, \ka(s_3)=\ktop, \ka(s_6)=\ka_f$, are given by:
\small
\begin{equation}
\label{deltas}
\begin{aligned}[c]
    s_1 &= \begin{cases}
    \frac{\sigmax}{\rhomax}, \ &\text{if } |\ktop -\ka_s|>\frac{(\sigmax)^2}{\rhomax}\\
     \sqrt{\frac{|\ktop -\ka_s|}{\rhomax}}, \ &\text{otherwise}
    \end{cases}\\
    s_2 &= \begin{cases}
    \frac{|\ktop - \ka_s|}{\sigmax},  \ &\text{if } |\ktop -\ka_s|>\frac{(\sigmax)^2}{\rhomax}\\
    s_1, \ &\text{otherwise}
    \end{cases}\\
    s_3 &=s_1 + s_2,\\
    s_4 &=\Delta + \begin{cases}
    \frac{\sigmax}{\rhomax}, \ &\text{if } |\ktop -\ka_f|>\frac{(\sigmax)^2}{\rhomax}\\
    \sqrt{\frac{|\ktop -\ka_f|}{\rhomax}}, \ &\text{otherwise}
    \end{cases}\\
    s_5 &= \Delta + \begin{cases}
    \frac{|\ktop - \ka_f|}{\sigmax},  \ &\text{if } |\ktop -\ka_f|>\frac{(\sigmax)^2}{\rhomax}\\
    s_4, \ &\text{otherwise}
    \end{cases}\\
    s_6 &=s_4 + s_5-\Delta.
\end{aligned}
\end{equation}
\normalsize
Also similar to~\cite{banzhaf2018g}, the curvature function of a $G3$ curve with $\rho(s)$ switching arc-lengths given by~\eqref{deltas} and with with initial curvature $\ka_s$, is given by
\begin{equation}
\label{ks}
    \ka(s)=\begin{cases}
    \ka_s \pm 0.5\rhomax s^2, \ s\in[0, s_1]\\
    \ka(\delta_1) \pm \rhomax s_1(s-s_1), \ s\in[s_1, s_2]\\
    \ka(s_2) \pm 0.5\rhomax(s-s_2)(s_2+2s_1-s), \ s\in[s_2, s_3]\\
    \ktop, \ s\in[s_3, \Delta]\\
    \ktop \mp 0.5\eta\rhomax(s-\Delta)^2, \ s\in[\Delta, s_4]\\
    \ka(s_4) \mp\eta \rhomax(s_4-\Delta)(s-s_4), \ s\in[s_4, s_5]\\
    \ka(s_5)\mp0.5\eta \rhomax(s-s_5)(s_5 -2\Delta+ 2s_4-s).
    \end{cases}
\end{equation}
Here, the top sign of each "$\pm, \mp$" is used if $\ktop\geq \ka_s$, while the bottom sign is used otherwise and $\eta=1$ if $\sigma(s_1)\sigma(s_4)<0$ and $\eta =-1$ otherwise. The value $\eta=1$ is used in circumstances where the curvature function increases to $\ktop$ and then decreases to $\ka_f$, or decreases to $\ktop$ and then increases to $\ka_f$. On the other hand, a value $\eta=-1$ is used when the curvature function is either monotonically increasing or decreasing over the entire curve.

 By~\eqref{dynamics}, the path state vector $\vect{p}$ along a $G3$ curve whose curvature is given by~\eqref{ks}, can be parameterized in terms of arc-length as: 
 \begin{equation}
     \label{configs}
     \vect{p}(s)=\begin{bmatrix}
     x(s)\\
     y(s)\\
     \z(s)\\
     \ka(s)
     \end{bmatrix} = \begin{bmatrix}
     x_s+\int_0^s\cos(\z(\tau)) d\tau\\
     y_s+\int_0^s\sin(\z(\tau)) d\tau\\
     z_0+z_1s+z_2s^2+z_3s^3\\
     \ka(s)
     \end{bmatrix},
 \end{equation}
 where $z_0, z_1, z_2, z_3$ are obtained by integrating $\ka(s)$ in~\eqref{ks}. As Figure~\ref{HCRcurve} implies, the point $\vect{x}_c$, and radius $r$, are given by
 \begin{equation}
\label{centers}
\vect{x}_c=\begin{bmatrix}
x_c\\
y_c
\end{bmatrix}=
\begin{bmatrix}
x(s_3)-\ktop^{-1}\sin(\z(s_3))\\
y(s_3)+\ktop^{-1}\cos(\z(s_3))
\end{bmatrix},
\end{equation}
\begin{equation}
\label{router}
    r=||(x(s_6), y(s_6)) - (x_c, y_c)||.
\end{equation}
Assuming that $\rhomax, \sigmax$ are known, we observe from the preceding arguments that the initial path states $\vect{p}_s$, the top and final curvatures $\ktop, \ka_f$ (resp.), and the arc-length $\Delta$ are enough to uniquely define a $G3$ curve. Thus we denote a $G3$ curve $G(\vect{p}_s, \ktop, \ka_f, \Delta)$ as the set of path states $\vect{p}(s)$ given by~\eqref{configs} from $s=0$ to $s=s_6$. A $G3$ path $P_{\rho}$, solving $\eqref{OP}$ is therefore a concatenation of such curves and straight lines where $\rho(s)$ is a function that takes only values in $\{\pm \rhomax, 0\}$.

As~\eqref{deltas} implies, the final arc-length of this curve is given by $s_6$. Let $G_i^{s_j}\varD$, be the value of path state $i$ at arc-length $s_j$. In the next section, we present our technique for solving \eqref{OP} by concatenating $G3$ curves and straight lines.

\subsection{Connecting $G3$ Curves}
This section proposes a technique to connect $G3$ curves using a straight line  via a reduction of the $G3$ curve path planning problem \eqref{OP} to a Dubin's-like path planning problem with two different minimal turning radii. The latter problem can be solved quickly by calculating the common tangents of two circles with different radii. For fixed values $\vect{p}_s, \ktop, \ka_f=0$, consider the set of curves
\be
\label{setofcurves}
\mathcal{G}=\{G(\vect{p}_s, \ktop, \ka_f=0, \Delta) \ : \ \Delta \geq s_3\}.
\ee
Observe that members of $\mathcal{G}$ are $G3$ curves with final curvature 0.  This is to facilitate connecting $G3$ curves using straight lines while preserving curvature continuity. The assumption that $\sigma(s_6)=0$ ensures that $G3$ curves can be connected with straight lines while preserving continuous curvature rate. 

For each pair $(\vect{p}_s, \ktop)$ we can construct a set $\mathcal{G}$ in \eqref{setofcurves}, containing the curve $G\varokmid$ as an element. Upon reaching curvature $\ktop$, the curvature of this curve immediately increases or decreases to $\ka_f$.  The curvature profile of such a curve can be found in Figure \ref{IncDelta}.

For fixed values $\vect{p}_s, \ktop, \ka_f=0$, the value of $s_3$ given in \eqref{deltas} is independent of $\Delta$, implying that every curve in $\mathcal{G}$ shares the same switching arc-length $s_3$. The same holds for the center $\vect{x}_c$ given in \eqref{centers}. Therefore, the values $\vect{p}_s, \ktop, \ka_f=0$, which are shared by all curves in $\mathcal{G}$, are sufficient to determine $s_3$ and $\vect{x}_c$. Finally, note that each curve in $\mathcal{G}$ is coincident with every other curve in $\mathcal{G}$ for all $s\leq s_3$. This phenomena can be see in Figure \ref{HCRcurve}:  the duration $\Delta-s_3$ that a curve spends at constant curvature $\ktop$ does not affect the portion of the curve between $\vect{p}_s$ and $\vect{p}(s_3)$.

\begin{figure}[t]
    \centering
    \includegraphics[width=0.8\linewidth]{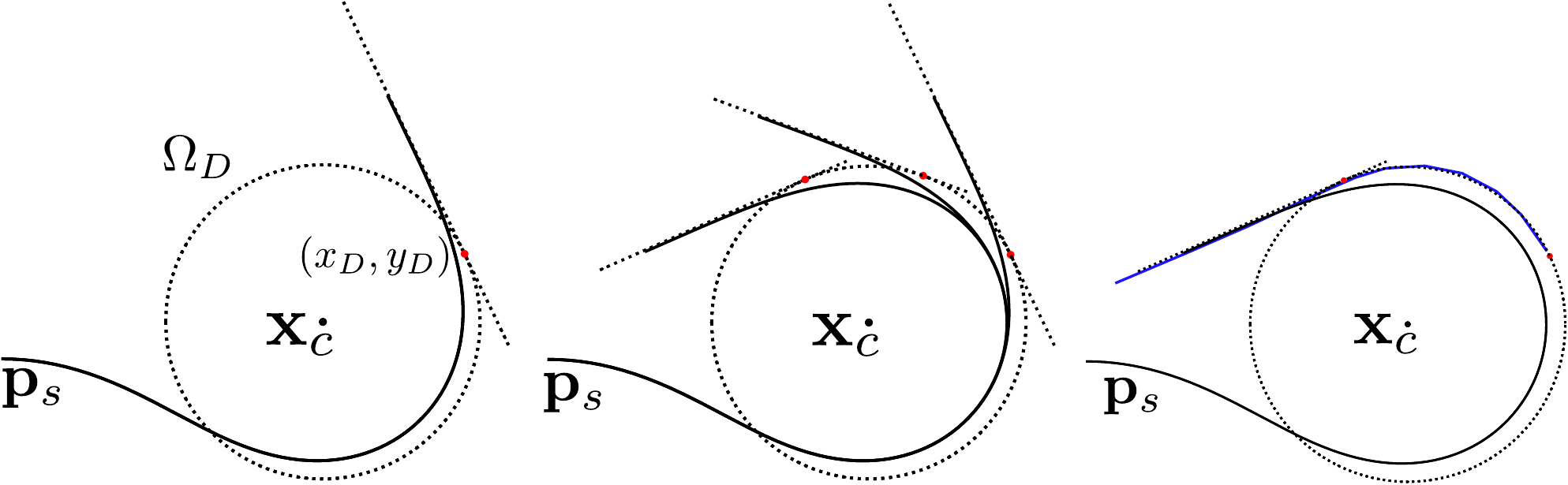}
    \caption{(Left) Representative circle $\Omega_D$, and point $(x_D, y_D)$. (Mid) Illustration of Theorem \ref{curve to circle}. (Right) Illustration of Step 5, with partial Dubin's path $\mathcal{D}$ (blue) and corresponding $G3$ curve $\hat{G}_i$ (black).}
    \label{Iris1}
\end{figure} 

The following Lemma illustrates the relationship between $\Delta$ and $G^{s_6}_{\z}\varD$. 

\begin{lemma}[Final Heading]
\label{LemmDelta}
For each set of curves $\mathcal{G}$ defined in \ref{setofcurves}, there is a unique solution $\Delta$ to the equation $G^{s_6}_{\z}\varD=\z_f$ such that $G\varD$ has minimum arc-length over $\Delta$:
\be
\label{Delta}
\Delta = s_3 + (\z_g-G_{\z}^{s_6}\varmid)\ktop^{-1}.
\ee
where $s_3$ is given by~\eqref{deltas} for the curve $G\varD$. 
\end{lemma}
The proof of Lemma \ref{LemmDelta} can be found in the Appendix. Lemma~\ref{LemmDelta} implies that the final headings of elements in $\mathcal{G}$ varies linearly with their values of $\Delta$. This Lemma motivates the following key Theorem that will be heavily leveraged later. 
\begin{theorem}[Reducing Theorem]
\label{curve to circle}
Given a set of curves $\mathcal{G}$ in \eqref{setofcurves}, let $x_f, y_f, \z_f$ denote the final $x,y$, and $\z$ values of any curve $G(\vect{p}_s, \ktop, 0, \Delta)\in\mathcal{G}$.  Also, let $m$ denote the slope of a line induced by $\z_f$, i.e., $m=\tan{\z_f}$.  The shortest distance from the point $\vect{x}_c$ defined in \eqref{centers} to the line passing through $(x_f, y_f)$ whose slope is $m$ is a constant with respect to $\Delta$.
\end{theorem}

The proof of Theorem \ref{curve to circle} can be found in the Appendix. The result of Theorem \ref{curve to circle} is shown in Figure \ref{Iris1} (Mid). In this figure, the red dots correspond to the points on the lines containing the final path states of each curve in $\mathcal{G}$ that are closest to $\vect{x}_c$. Observe that they lie on a circle centered at $\vect{x}_c$. This theorem has the following major consequence that allows us to reduce a $G3$ path planning problem to a Dubin's-like path planning problem if $\vect{p}_s, \vect{p}_g, \ktop, \sigmax, \rhomax$ are fixed, and only $\Delta$ may vary. This interim result is the foundation of our proposed $G3$ path planning technique. We begin with a definition.

\begin{definition}[Representative Circle]
Given a set of curves $\mathcal{G}$ in \eqref{setofcurves}, and using the same notation as Theorem \ref{curve to circle}, the \emph{representative circle} $\Omega_D$ of a curve in $\mathcal{G}$, is the circle centered at $\vect{x_c}=(x_c, y_c)$ containing the point $(x_D, y_D)$ where
\begin{equation}
\label{xDyD}
    \begin{split}
        x_D =& \begin{cases} \frac{(mx_f - y_f+m^{-1}x_c+y_c)}{m + m^{-1}}, \ &\text{if } \z_f\neq \pi/2\\
        x_f, \ &\text{otherwise}
        \end{cases}
        \\
        y_D =&
        \begin{cases}
        -m^{-1}(x_D - x_c)+y_c, \ &\text{if } \z_f\neq \pi/2\\
        y_c \ &\text{otherwise}.
        \end{cases}
    \end{split}
\end{equation}
That is, for the straight line containing $(x_f, y_f)$ whose slope is $\tan \z_f$, the representative circle is the circle centered at $\vect{x}_c$ containing the point on the line closest to $\vect{x}_c$.  Observe that Theorem \ref{curve to circle} implies that all curves in a set $\mathcal{G}$ have coincident representative circles. An example representative circle, and corresponding points $(x_D, y_D)$ is show in Figure~\ref{Iris1} (Left). In this figure, the black curve is $G\varokmid$.  Theorem \ref{curve to circle} implies that every $G3$ curve in a given set $\mathcal{G}$ in \eqref{setofcurves} has the same representative circle.  This fact is leveraged in the following section.
\end{definition}
\subsection{Connecting $G3$ curves with a Straight Line}
 We are now prepared to present our algorithm for computing $G3$ paths that solve \eqref{OP} given a known value of $\rhomax$. The high-level idea is to begin by constructing two $G3$ curves originating from both start configuration and goal configuration (with reverse orientation), and then to connect these curves using either a straight line or another curve.  For simplicity, let $\vect{p}_1=\vect{p_s}, \vect{p}_2=\vect{p}_g$, and let $\vect{p}_3$ denote $\vect{p}_g$ with reverse orientation. That is, $\vect{p}_3=(x_g, y_g, \z_g+\pi, -\ka_g)$.
\begin{description}[style=unboxed,leftmargin=0cm]
\vspace{3mm}
\item[Step 1:] For each of the possible start and goal curve orientations: (R,L), (R,R), (L,R), (L, L) where R is a right-hand turn and L a left, perform the following steps, and then discard all but the choice with minimum final arc-length.
\vspace{3mm}
\item[Step 2:]  Let $G_i=G(\vect{p}_i, |{\ktop}_i|=\kamax, \ka_{f_i}=0, \Delta_i=s_{3_i}),$ $i=1,3$ be two curves originating from $\vect{p}_1, \vect{p}_3$, respectively, whose top curvature reaches its maximum allowable magnitude $\kamax$ (with sign given by the chosen orientation), and where $s_{3_i}$ is given by \eqref{deltas}. This step is illustrated by the black curves in Figure \ref{FirstFewSteps} (Left) for orientation (R,L).  Let $\mathcal{G}_i$ denote the two sets of curves from \eqref{setofcurves} with $G_i\in\mathcal{G}_i$.
\vspace{3mm}
\item[Step 3:] Compute the points $(x_{D_i}, y_{D_i}), i=1,3$ for curves $G_i$ on the representative circles of each curve. This step is illustrated in Figure \ref{FirstFewSteps} (Left). Here, the dotted circles represent the representative circles of each curve, and the red dots on these circles represent the points $(x_{D_i}, y_{D_i}), i=1,3$.
\vspace{3mm}
\item[Step 4:] Compute the Dubin's-type path $\mathcal{D}$ that consists of two curves and a straight line connecting $(x_{D_i}, y_{D_i}), i=1,3$ with minimum turning radius on each curve $r_{D_i}=||(x_{D_i}, y_{D_i})-(x_{c_i}, y_{c_i})||$. This step is illustrated in Figure \ref{FirstFewSteps} (mid). Note that curves $G_i, i=1,3$ may have different values of $r_{D_i}$, resulting in a Dubin's problem with different minimum turning radii. This may not have a solution of the form described here, which will be addressed this in the next section.
\vspace{3mm}
\item[Step 5:] Compute the angle $\z_f$ inscribed by the straight portion of $\mathcal{D}$ and the horizontal axis.  Compute the value of $\hat{\Delta}_i$ from \eqref{Delta} such that the curves $\hat{G}_i=G(\vect{p}_i, {\ktop}_i, \ka_{f_i}, \hat{\Delta}_i)\in\mathcal{G}_i$ $i=1,3$ have final headings $\z_f, \z_f+\pi$, respectively.  By Theorem \ref{curve to circle}, curves $G, \hat{G}_i$ have coincident representative circles which, by construction, have radii $r_{D_i}$ the minimum turning radii of $\mathcal{D}$.  Therefore, the final states $(x_{f_i}, y_{f_i}, \z_{f_i}, \ka_{i}=0,\sigma_{f_i}=0)$ of the curves $\hat{G}_i$ are on the path $\mathcal{D}$. This is illustrated in Figure \ref{Iris1} (Right).  Connect the curves $\hat{G}_i=G(\vect{p}_i, {\ktop}_i, \ka_{f_i}=0, \Delta_i)$ with a straight line. This step is illustrated in Figure \ref{FirstFewSteps} (Right). Solid black, and dotted black lines represent $\hat{G}_i, G_i$, respectively.
\vspace{3mm}
\item[Step 6: (Looping)] If $\sigmax, \rhomax$ are too small compared to $\kamax$, then one or both of the curves $G_i$ in Step 2 will \emph{loop} \cite{banzhaf2018g}. A curve $ G(\vect{p}_s, \ktop, \ka_f, \Delta)$ with initial and final headings $\z_s, \z_f$ (resp.), will loop if $|\z_f -\z_s|< G_{\z}^{s_6}(\vect{p}_s, \ktop, \ka_f, \Delta=s_3) $. If such looping occurs on curve $i=1,3$, we decrease the magnitude of the top curvature ${\ktop}_i$, and return to Step 2. Minimum arc-length is guaranteed by computing the smallest magnitude values of ${\ktop}_i$ that ensure a connection in Step 5.
\end{description}

\subsection{Connecting $G3$ curves with a $G3$ curve}
Two problems may arise in Steps 1-6 above. The first, we call \emph{overlap}. Overlap arises when the endpoints of the two curves $\hat{G}_i, i=1,3$ computed in Step 5 cannot be connected with a straight line that preserves orientation. The phenomenon is illustrated in Figure \ref{overlap} (Left), and can occur if the centers $\vect{x}_{c_i}$ of the representative circles of the curves $G_i$ in Step 2 are too close together ($< r_1 + r_3$ where $r_i$ is the radius \eqref{router} for the curve $\hat{G}_i$). Observe that overlap is a result of $\sigmax, \rhomax$ being small relative to ${\ktop}_i$: by the time the curvature of curve $\hat{G}_i$ changes from ${\ktop}_i$ to 0, the final configuration has already passed that of the second curve and cannot be connected with a straight line that preserves orientation. Therefore, if overlap occurs we attempt to find a curvature other than 0 that can be used to connect $\hat{G}_i$. This is summarized here:
\begin{description}[style=unboxed,leftmargin=0cm] 
\vspace{3mm}
\item[Step 7: (Overlap)]  If overlap occurs, we cut the curves $\hat{G}_i, i=1,3$ at $s=\Delta_i$. That is, let $\hat{G}_1^{\text{cut}} = G(\vect{p}_i, {\ktop}_i, \ka_{f_i}={\ktop}_i, \Delta_i)$, and we attempt to connect $\hat{G}_1^{\text{cut}}, i=1,3$ with a third $G_3$ curve. If no such third curve exists, we slowly decrease the value of $\sigmax$, and go back to Step 1. Figure \ref{overlap} (right) illustrates an example where $\sigmax$ is decreased until $\hat{G}_i^{\text{cut}}, i=1,3$ can be connected with a third curve.  Minimum arc-length is preserved by finding the largest feasible value of $\sigmax$ that allows for such a connection.
\end{description}
\vspace{3mm}
The second problem arises if no Dubins' solution can be found in Step 4. Similar to overlap, this arises if the centers $\vect{x}_{c_i}$ are too close together. For this reason, we again decrease the value of $\sigmax$ and return to Step 1. At the end of Steps 1-7, the process described above will have computed a $G3$ curve $P_{\rhomax}(s)$ that solves $\eqref{OP}$.  
\begin{figure}
    \centering
    \includegraphics[width=0.8\linewidth]{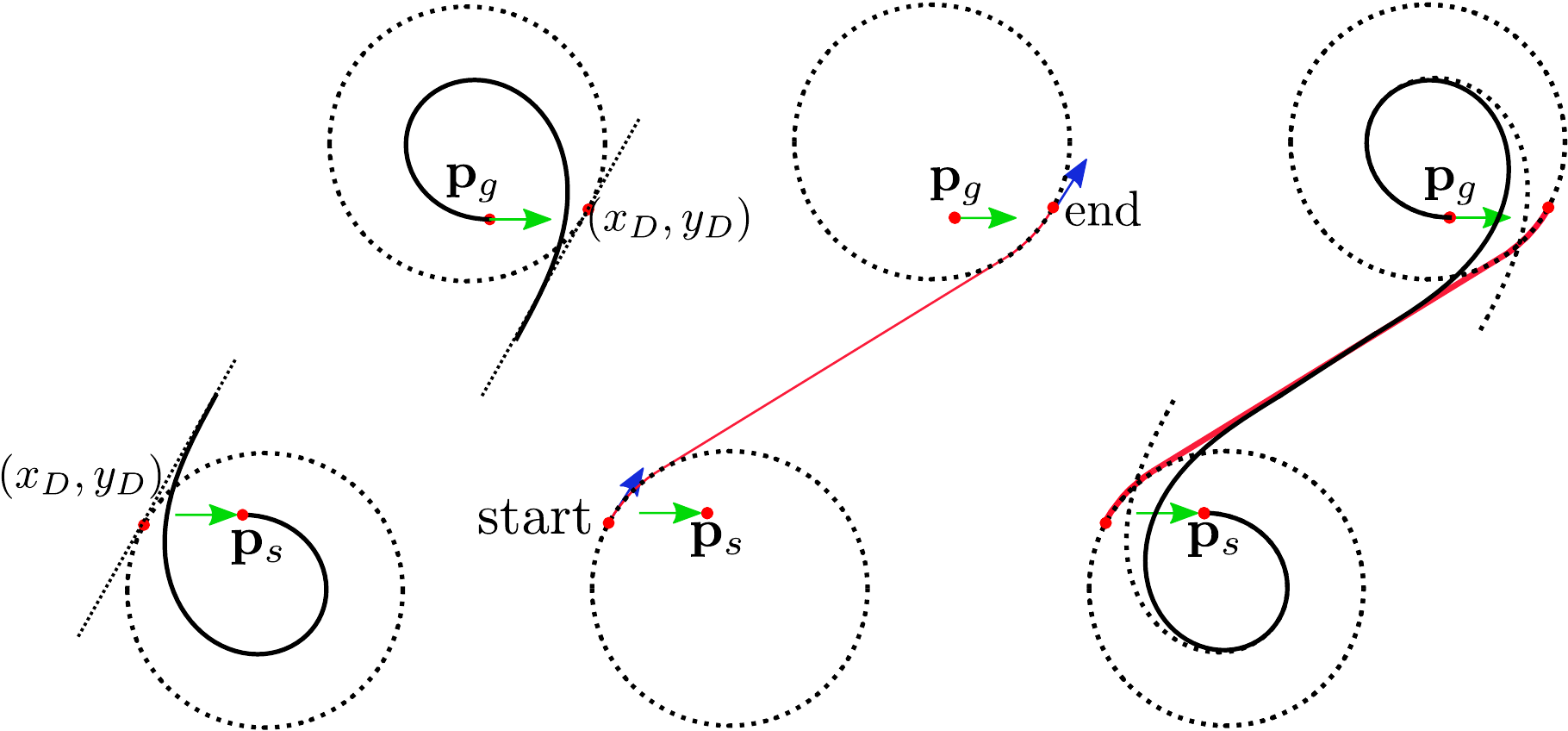}
    \caption{(Left) Representative circles around $\vect{p}_s$ and $\vect{p}_g$ in reverse. (Mid) Dubins'-like solution between start and goal states. (Right) Solution path.}
    \label{FirstFewSteps}
\end{figure}

\begin{figure}[t]
    \centering
    \includegraphics[width=0.8\linewidth]{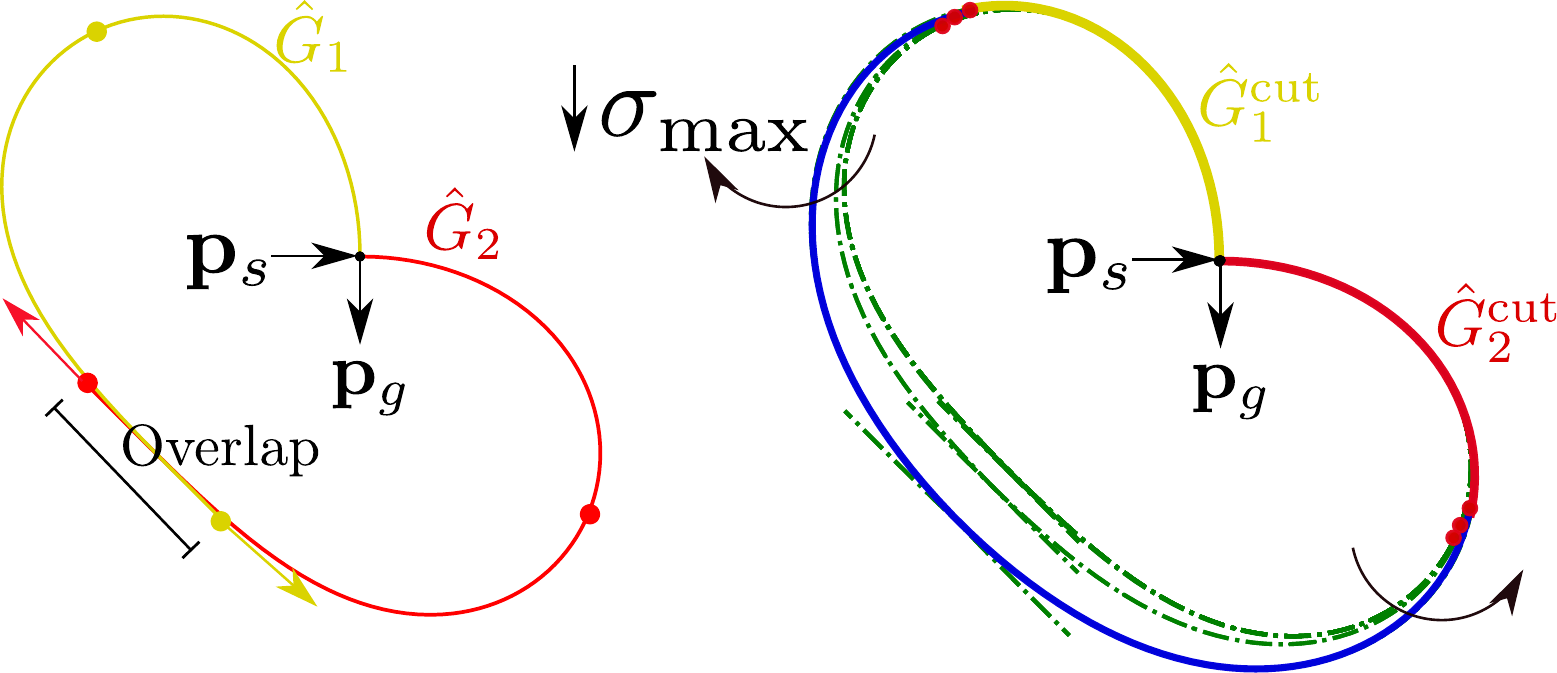}
    \caption{Step 7. (Left) Illustration of overlap. start and goal configurations too close together resulting in curve $P2$ terminating behind $P1$. (Right) Stretching of curves $P1, P2$ by decreasing $\sigmax$ until $P1, P2$ can be connected by third curve.}  
    \label{overlap}
\end{figure}

\section{Computing Velocity Profiles}
\label{VelSec}
In this section we describe how to compute a velocity profile $v(s)$ that minimizes the cost function in \eqref{CostFunc} over a given $G3$ path. At a high level, the process is best described as an instance of Sequential Quadratic Programming (SQP). Similar to the procedure outlined in \cite{biegler1984solution}, we change the continuous optimization problem of computing an optimal velocity profile given a curve, to a discrete problem that asymptotically approaches the continuous version. Recall that we wish to compute $\beta^*=\min_{\beta\in B}C(P_{\rho}(s), v_{\beta}(s))$ where $P_{\rho}(s)$ is known and fixed and has final arc-length $s_f$ (see \eqref{OPAPPROX}). To wit, we make two assumptions: first, for sufficiently large $n\in\mathbb{N}_{\geq 3}$, the value of $v_{\beta^*}(s)$ is given as a polynomial in $s$ of degree no more $n+1$. Second, that for $n-1$ sampled arc-lengths $\{s_2>0, s_3,\dots, s_n=s_f\}$ in the interval $[0, s_f]$, the function $\alpha^*(s)=dv^*(s)/ds$ takes values $\alpha^*(s_i)=a_i^*, i=2,\dots,n$ for $\vect{a}^*=[a_2^*,\dots, a^*_n=0]$. These assumptions allow us to discretize the problem as follows:  given any decision vector $\vect{a}=[a_2, a_3,\dots, a_n=0]$, let 
\be
\label{vamodel}
\begin{split}
\alpha^{\vect{a}}(s)&= \sum_{i=1}^{n}p_is^i, \\ 
v^{\vect{a}}(s)& = v_s +  \sum_{i=1}^{n}\frac{p_i}{i+1}s^{i+1}, \ \beta^{\vect{a}}(s)=\sum_{i=1}^nip_is^{i-1}
\end{split}
\ee
where
\be
\label{lineqvel}
\begin{bmatrix}
p_1\\
p_2\\
\vdots\\
p_n
\end{bmatrix}=V^{-1}\begin{bmatrix}
v_f - v_s\\
a_2\\
\vdots\\
a_n=0
\end{bmatrix}
\ee
\be
\label{VM}
V = \begin{bmatrix}
s_n^2/2 &s_n^3/3 &s_n^4/4 &\dots & s_n^{n+1}/(n+1)\\
s_2 & s_2^2 & s_2^3 &\dots & s_2^{n}\\
s_3 & s_3^2 & s_3^3 &\dots & s_3^{n}\\
\vdots &\vdots &\vdots &\ddots &\vdots\\
s_{n} & s_{n}^2 & s_{n}^3 &\dots & s_{n}^{n}
\end{bmatrix}.
\ee
Observe that $v^{\vect{a}}(s), \alpha^{\vect{a}}(s)$ in \eqref{vamodel} satisfy the boundary (if not magnitude bound) constraints in \eqref{NEWOP}. Indeed, from \eqref{vamodel}, it is clear that $\alpha^{\vect{a}}(0)=0, v^{\vect{a}}(0)=v_s$. Moreover, the first and last equations in the system of linear equations $V\times [p_1, \dots, p_n]^T=[v_f-v_s, a_2, \dots, a_n=0]$ ensure that $v^{\vect{a}}(s_f)=v_g, \alpha^{\vect{a}}(s_f)=0$, respectively. The remaining $n-2$ equations in this system ensure that $\alpha^{\vect{a}}(s_i)=a_i, i=2,\dots, n-1$. Therefore, the above equations provide a means of  constructing functions $v^{\vect{a}}(s),\alpha^{\vect{a}}(s), \beta^{\vect{a}}(s)$ that satisfy the boundary requirements of \eqref{NEWOP} and $\alpha^{\vect{a}}(s_i)=a_i, i=2,\dots,n$ for any choice $\vect{a}$. Observing that $V$ in \eqref{VM} is independent of $\vect{a}$, we note that if the sampled arc-lengths $s_i, i=2,\dots, n$ are fixed across all choices of $\vect{a}$, then any iterative process to minimize cost over samples $\vect{a}$ would only need to compute $V^{-1}$ once. Moreover, we note that $V$ is guaranteed to be non-singular provided that $s_i\neq s_j, \forall i,j=2,\dots,n$. 

We now describe how to select ${\vect{a}}$ that minimizes the cost $C(\vect{a})=C(P_{\rho}(s), v^{\vect{a}}(s))$. This solution provides the velocity profile $v^{\vect{a}}(s)$ required in \eqref{OPAPPROX}. For fixed path $P_{\rho}(s)$, let

\be
\begin{split}
\mathcal{L}\big(\vect{a}\big) = &w_a\tilde{C}_a+w_{\J}\tilde{C}_{\J} + w_y\tilde{C}_y+w_t\tilde{C}_t\\
&+\delta(\alpha^{\vect{a}}(s), \alpha_{\text{max}}) + \delta(v^{\vect{a}}(s), v_{\text{max}}) + \delta(\beta^{\vect{a}}(s), \beta_{\text{max}})
\end{split}
\ee

where $\delta(x, x_{\text{max}}) =  (\max(|x| - x_{\text{max}}, 0))^{M} $ for large $M$, is a boundary function penalizing values of $|x|$ above $x_{\text{max}}$, and
where $\tilde{C}_a, \tilde{C}_y,\tilde{C}_{\J},\tilde{C}_t$ are given in \eqref{NEWCOST}. Observe that for sufficiently large $M$, and under the assumptions of this section, 
\be
\label{onestep}
\begin{split}
    \beta^*(s)=\underset{\beta\in B}{\text{argmin}} \ C(P_{\rho}, v_{\beta}) \ \iff\\
    \begin{cases}
    \beta^*(s)&=\beta^{\vect{a}^*}(s),\\
    \vect{a}^*&=\underset{\vect{a}^*\in \mathbb{R}^{n-1}}{\text{argmin}} C(\vect{a}),\\
    C(\vect{a}) &=  \int_{0}^{s_f}\mathcal{L}(\vect{a})ds
    \end{cases}.
\end{split}
\ee
Indeed, the assumption in this section, is that for sufficiently large $n$, the cost-minimizing velocity profile $v^*(s)$ can be expressed as a polynomial of degree $n+1$. Therefore, $\beta^*(s)$ is a polynomial of degree $n-1$ implying that there must exist a vector $\vect{a}^*$ with $\beta^*(s)=\beta^{\vect{a}^*}(s)$ in accordance with \eqref{vamodel}-\eqref{VM}. Next, we observe that for sufficiently large $M$, a choice of $\vect{a}$ will minimize $C(\vect{a}) =  \int_{0}^{s_f}\mathcal{L}(\vect{a})ds$ only if $\beta^{\vect{a}}(s)\in B$. Therefore, to determine the function $\beta^*(s)$, for sufficiently large values of $n, M$, we need only solve \eqref{onestep}. This can be done using SQP. Finally, observing that \eqref{onestep} is unconstrained, SQP reduces to gradient decent allowing us to use Python's built-in SQPL minimization function.

\section{Results}
We now demonstrate the benefits of our approach. We begin by illustrating how the techniques developed here can be used to anticipate the driving styles of several archetypal users, namely, a comfort-favoring (comfort) user who prefers low speed and high comfort, a moderate user who favors a mix of speed and comfort, and a speed-favoring (speed) user. Next, we compare the theoretical cost of paths generated using the proposed method to those from~\cite{banzhaf2018g}. As implied by Theorem \ref{G3asOpt}, any path that solves \eqref{NEWOP} is a $G3$ path. It is for this reason that we compare our methods for generating trajectories to those proposed in~\cite{banzhaf2018g}, the current state-of-the-art $G3$ path generation technique.
Finally, we measure tracking error and final cost of our proposed trajectories using Matlab's seven degree-of-freedom simulator and built-in Stanley controller~\cite{hoffmann2007autonomous}. The techniques described here were encoded entirely in Python 3.7 (Spyder), and the results were obtained using a desktop equipped with an AMD Ryzen 3 2200G processor and 8GB of RAM running Windows 10 OS.
\subsection{Setup}

\begin{description}[style=unboxed,leftmargin=0cm]
\item[Unit-less Weighting]  The features in the cost function  \eqref{NEWCOST} represent different physical quantities. In order for our cost function to represent a meaningful trade-off between these features for given weights, a scaling factor is included in the weight~\cite{hogan2009sensitivity}. Similar to the scaling techniques in~\cite{gulati2013nonlinear}, we let $\hat{C}_{m}, m\in\{a, t, y, \J\}$ denote the cost of the trajectory that solves \eqref{NEWOP} given weight $w_{m}=1$, $w_{n}=0$, $\forall n\in \{a, t, y, \J\}$, $n\neq m$. We then scale the weights as
$$
\hat{w}_{m}=\frac{w_m}{\hat{C}_m}\sum_{ m\in\{a, \J, y, t\}}\hat{C}_m, \ \ m\in\{a, \J, y, t\}.
$$
This scaling process\footnote{NOTE: The scaling process described here is not included in the timing of our algorithm reported in Table \ref{Runtime}} has the following benefits: first, if all features are weighted equally, then the scaled cost of each feature $\hat{w}_m\hat{C}_m$ will be equal. Second, it may be the case that for some $m\in \{a, \J, y, t\}$, the feature cost $C_m$ is very sensitive to changes in the trajectory in a neighborhood around the optimal trajectory. For example, IS jerk is the integral of a fifth order polynomial of $v$ in \eqref{NEWCOST} and is therefore highly sensitive to changes in $v$ at high speeds,  In these cases, if the scaling factor used any but the optimal value $\hat{C}_m$ for each feature, it would risk of over-reducing the weight on this feature.
\\
\item[Parameter Bounds] For all experiments, the curvature parameter limits are given by~\cite{banzhaf2018g}:

\begin{equation}
    \kamax = 0.1982 m^{-1}, \  \sigmax = 0.1868 m^{-2} \ 
    \rhomax = 0.3905 m^{-3}. 
\end{equation}

while the limits on velocity, acceleration, and instantaneous jerk are given, respectively, by~\cite{bae2019toward}

\begin{equation}
    \vmax = 100 km/hr, \ a_{\text{max}} = 0.9m/s^2, \ b_{\text{max}}=0.6 m/s^3.
\end{equation}

\subsection{Evaluation}
\item[Qualitative Analysis] To qualitatively evaluate our methods, we analyse the trajectories for three archetypal users: a speed user, who prefers low travel time over comfort, an intermediate user who values comfort and speed equally, and a comfort user who emphasizes comfort over speed. To simplify the comparison of these users, we combine the scaled but unweighted features IS acceleration, IS jerk, and IS yaw into a \emph{discomfort cost}, and compare this to \emph{time cost} for each user:
\begin{equation*}
    \begin{split}
        \text{discomfort cost} = \sum_{m\in \{a, \J, y\}}\hat{w}_m\hat{C}_m, \ \text{time cost} = \hat{w}_m\hat{C}_t.
    \end{split}
\end{equation*}
Given the simplistic nature of the speed, intermediate, and comfort archetypes, the expected form of their favored trajectories is apparent: the speed user wishes to minimize travel time at the expense of comfort. Therefore, she will favor trajectories featuring short paths and high velocities. The comfort user, on the other hand, prefers comfortable trajectories over short travel times. Finally, the intermediate driver should favor a balance of the two other users. In this experiment, the start-goal configurations for each user are:
$$
p_s = [0, 0, 0, 0], \ p_f=[30, 105,0, 0.1695],
$$
with $v_s=v_g = 7m/s$. The weights for each user are:
\begin{equation}
    \begin{split}
        \text{speed User} \ &w_t=0.9, w_y=w_{\J}=w_a = 0.0333,\\
        \text{Intermediate User} \ &w_t=w_y=w_{\J}=w_a=0.25\\
        \text{comfort User} \ &w_t=0.01, w_y=w_{\J}=w_a=0.33.
    \end{split}
\end{equation}
Thus, the speed user places a high weight on time, and a small weight on IS yaw rate, IS Jerk, and IS acceleration (i.e., small weight on discomfort), while the comfort user places a low weight on time, and the intermediate user places an equal weight on all features. 
\begin{figure}[t]
    \centering
    \includegraphics[width=\linewidth]{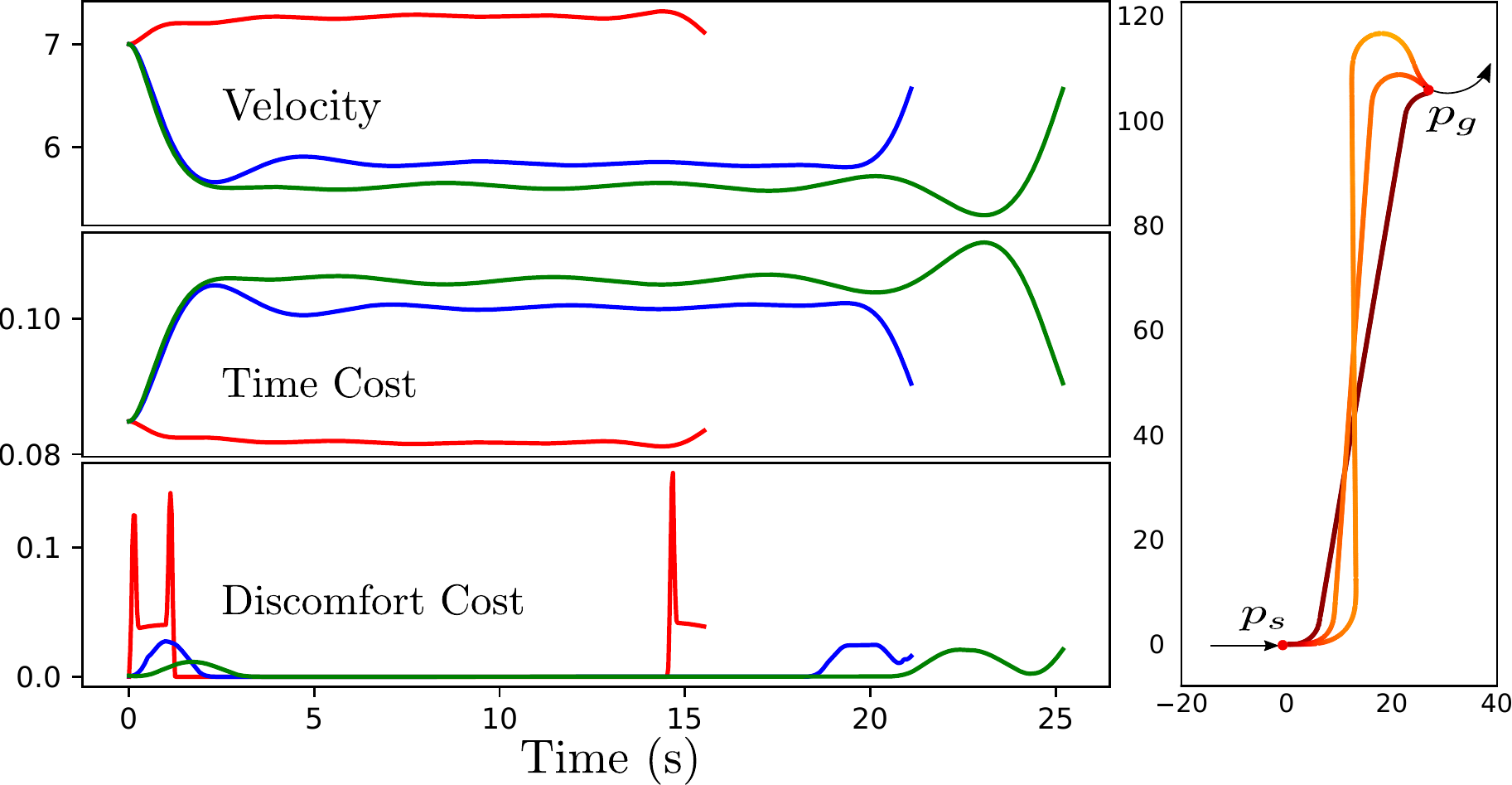}
    \caption{Comparison of three users. Left: velocity (m/s) (top), time cost (mid) and discomfort cost (bottom) for three users. Red, blue green curves represent speed, intermediate, and comfort users respectively  Right: trajectories of the tree users. The longest trajectory is that of the comfort user, the shortest is that of the speed. Here, color represents velocity.}  
    \label{Qualitative}
\end{figure} 

The results are illustrated in Figure~\ref{Qualitative}. The left image illustrates the velocity profile (top), time cost (mid), and discomfort cost (bottom) for the speed (red), intermediate (blue), and comfort (green) drivers. The right image illustrates the resulting trajectories for each user with color representing velocity. The shortest trajectory is that of the speed user. Observe that her path is short and velocities high, reflecting her preference towards short travel times. The longest trajectory is that of the comfort user, and we note that her trajectory features consistently low velocities and a long path. This tendency towards longer path length, allows the comfort user to change velocities over a longer time period allowing her to reach lower velocities with lower acceleration and jerk. This is reflected in the low discomfort cost of the comfort user. The middle trajectory is that of the intermediate user, who decreases her velocity on the turn to minimize jerk/acceleration, but then speeds up again on the straight portion of the path. 

A further example of the three archetypes is given in Figure \ref{PNWHL} in which a lane change maneuver is preformed from $(0, 0, 0, 0)$ to $(50, 6, 0, 0)$ with an initial and final velocity of  10 m/s. Cars are drawn ever 0.75 seconds, with the dotted vertical lines representing horizontal position of the top-most car at each time step. The comfort driver has an average speed of 9.4 m/s and completes the maneuver in 5.30 seconds, while the mixed and speed drivers have average speeds and final times of (10.09 m/s, 4.91 seconds), and (10.80 m/s, 4.55 seconds), respectively. Observe that the car following the red trajectory arrives at the goal almost a full time-step ahead of the blue-trajectory car. 

\vspace{3mm}

\item[Path Evaluation] We now evaluate the quality of the paths generated using our technique. To this end, we isolate the effect of the path on the final cost by comparing two methods for generating trajectories:
\begin{itemize}
    \item \emph{Ours}:  Paths and velocity profiles are computed using the methods detailed in this paper.
    \item \emph{Benchmark}: velocity profiles are computed using the techniques outlined here, but paths are computed from \cite{banzhaf2018g}.
\end{itemize}
We define the $\%$ Savings as the difference in cost (given in \eqref{NEWCOST}) for trajectories generated using Benchmark and Ours as a percentage of the cost of Benchmark-based trajectories.

For this experiment, 1300 start-goal pairs were randomly generated with initial and final curvatures ranging from $-\kamax$ to $\kamax$, initial and final velocities (taken to be equal) ranging from $0$ to $\vmax$, initial and final $x,y$-values ranging from $-20m$ to $20m$, and weights ranging from 0 to 1. Initial and final velocities were chosen to reflect the Euclidean distance between start-goal pairs, and pairs with no feasible solution were disregarded. The results of these experiments are tabulated in Table \ref{PathSave}. Here, a feature (IS time, IS yaw, IS acceleration, IS jerk) is said to be \emph{dominant} if the corresponding weight is greater than 0.5. If no weight is greater than 0.5, the features are said to be \emph{blended}. It was found that the average percent savings was $33.10\%$ across all experiments. Moreover, if the velocity is not optimized using our method, but is assumed to be constant instead, we see see a $45.5\%$ savings. As Table \ref{PathSave} implies, the largest savings is typically found when IS jerk is the dominant feature, and when initial and final velocities are medium to high. This implies that the techniques presented here have greater value in highway driving scenarios (where velocities are high, and high values of jerk can lead to dangerous maneuvers) than driving in a parking-lot, but can still reduce costs by around $25\%$ on average. A detailed view of the distribution of the percent savings for the experiments is given in Figure \ref{DISTR}. Observe that the instances where the percent savings are substantial $(>50\%)$ are not infrequent. In particular, in $32\%$ of experiments, the savings was at least $50\%$, while in $26\%$ of experiments, the savings was at least $70\%$, and $17\%$ of experiments had a savings of at least $80\%$.

\begin{table}
\centering
\begin{tabular}{@{}   l l c @{}} 
\toprule
 & & Avg. Percent Savings (\%) \\
\midrule
\multirow{5}{4em}{Dominant Feature} & IS Time & 23.04 (31.4)\\
& IS Yaw &  28.69 (32.4) \\
& IS Acceleration & 27.66 (32.5)\\
& IS Jerk & 42.81 (35.1)\\
& Blended & 39.36 (35.9)\\
\midrule
\multirow{3}{4em}{Initial/Final Velocity} & Low $(0,  \vmax/3]$ & 24.38 (31.9)\\
& Med $(\vmax/3, 2\vmax/3]$ & 41.74 (35.5)\\
& High $(2\vmax/3, \vmax]$ & 39.71 (34.3)\\
\bottomrule
\end{tabular}
\caption{Average cost savings. Breakdown by dominant feature and initial/final velocity.}
\label{PathSave}
\end{table}

\begin{figure}[t]
    \centering
    \includegraphics[width=0.8\linewidth]{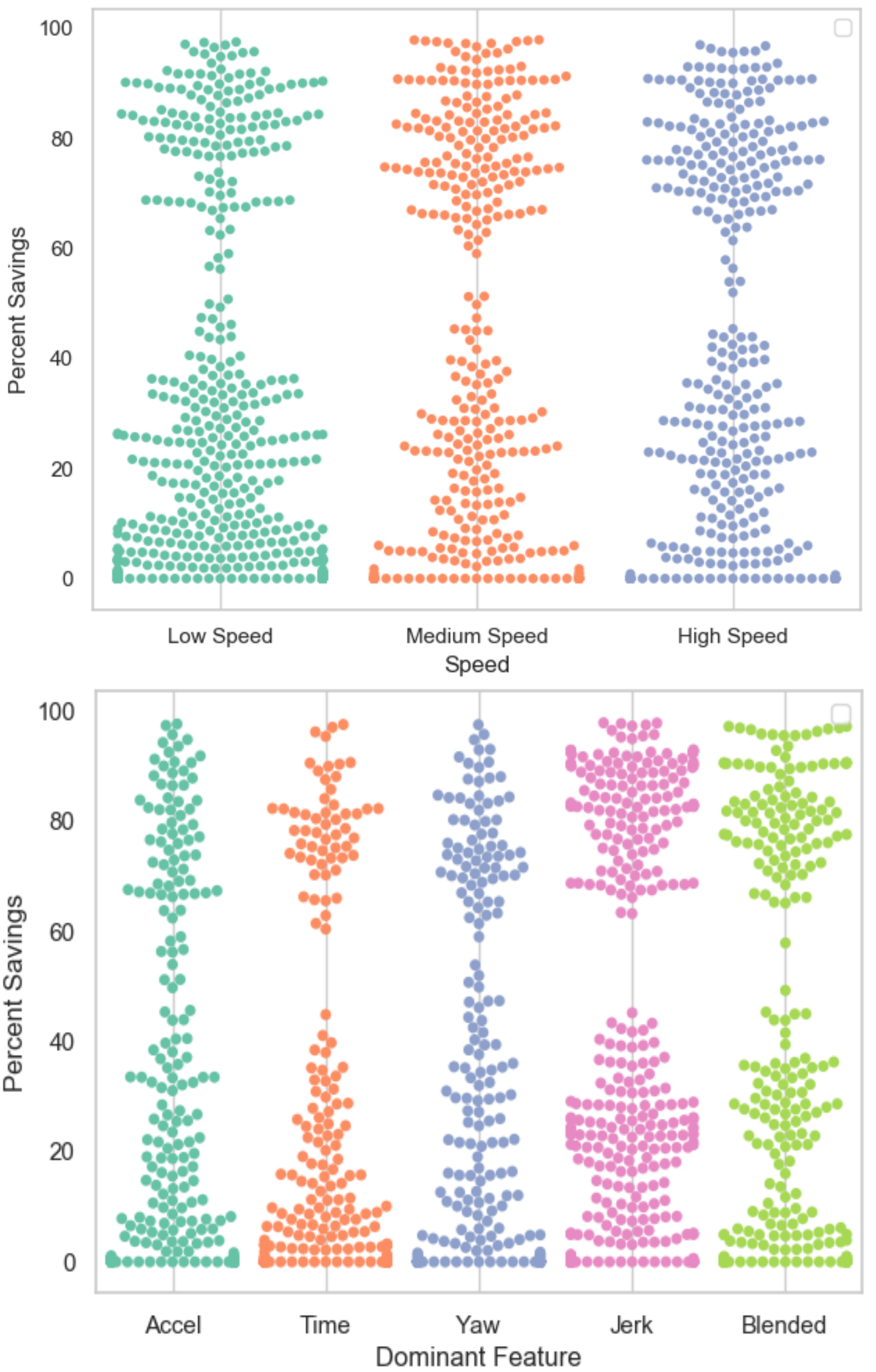}
    \caption{Percent savings distribution by initial/final velocity (top), and dominant feature (bottom).}  
    \label{DISTR}
\end{figure}

\begin{figure}[t]
    \centering
    \includegraphics[width=0.8\linewidth]{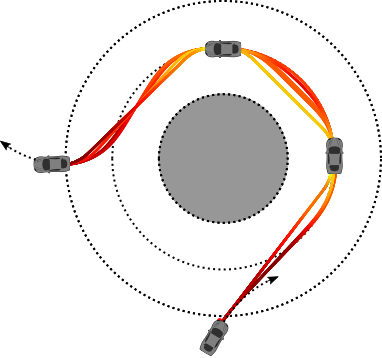}
    \caption{Optimal trajectories for weights between way-points of a roundabout. Cars represent fixed way-points (position, curvature, velocity) while color gradient of each trajectory represents velocity.}  
    \label{Roundabout}
\end{figure} 
\begin{figure}[t]
    \centering
    \includegraphics[width=\linewidth]{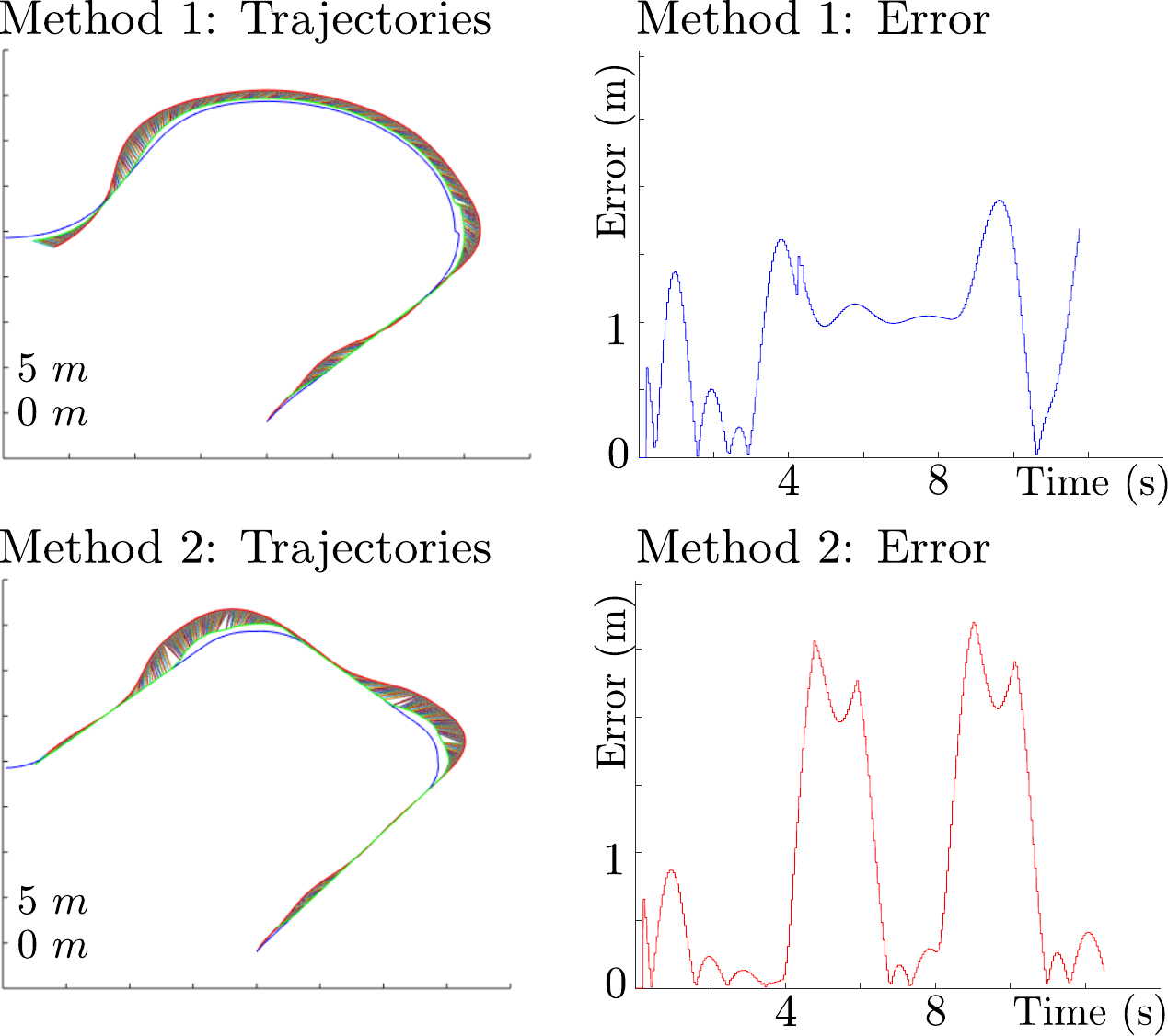}
    \caption{Simulated trajectories (left) and Lateral Error over time (right) for example roundabout maneuver using Methods 1 (Ours) and 2 (Benchmark) (top and bottom, respectively).}  
    \label{MatlabTraj}
\end{figure} 
\begin{figure}[t]
    \centering
    \includegraphics[width=0.8\linewidth]{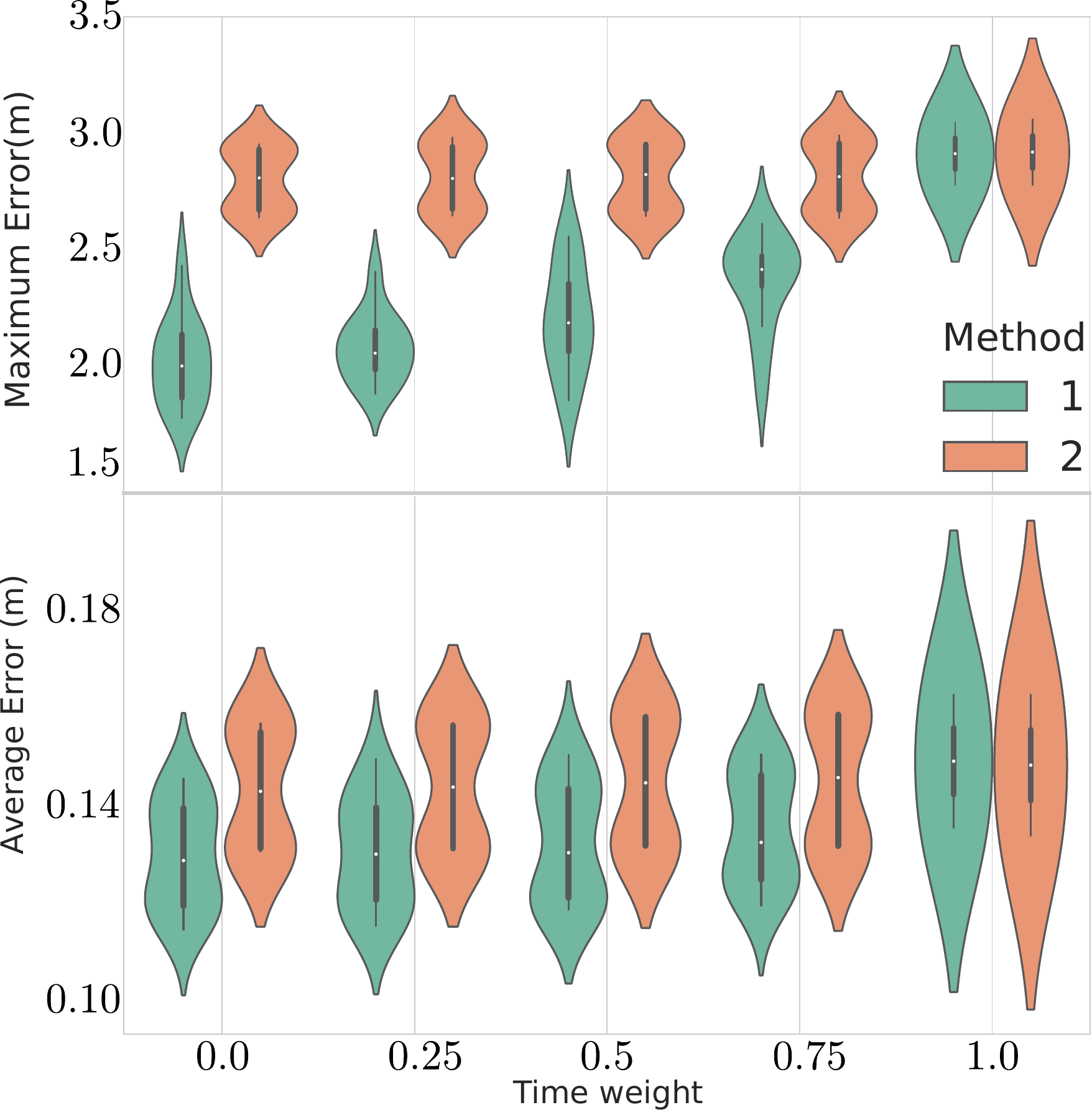}
    \caption{Average maximum errors (top) and average average errors (bottom) for Methods 1 and 2 categorized by time weight.}  
    \label{MatlabData}
\end{figure} 

\begin{table}
\centering
\begin{tabular}{@{}   c c c c c c @{}} 
\toprule
Time weight ($w_t$) & Cost Savings (\%) & \multicolumn{2}{c}{Max Err (m)} & \multicolumn{2}{c}{Avg Err (m)}\\
\cmidrule{3-4}
\cmidrule{5-6}
& & O & B & O & B\\
\midrule
0 &33.6 &2.05 &2.83 &0.131 &0.145\\
0.25 &19.7 &2.11 &2.84 &0.132 &0.145\\
0.5 & 12.2 &2.23&2.85&0.134&0.147\\
0.75 &8.4 &2.40&2.85&0.137     &0.147\\
1 &0.35 &2.95	&2.95&0.151     &0.150\\
\bottomrule
\end{tabular}
\caption{Average cost savings, maximum lateral error and average lateral error for two methods (O = Ours, B = Benchmark), categorized by time weight.}
\label{PathSaveMAT}
\end{table}

\vspace{3mm}

\item[Tracking Error and Simulation Cost] The above results imply that the methods proposed herein can be used to compute reference trajectories that result in substantial cost reduction if tracked perfectly by a vehicle.  However, it is highly unlikely that any reference will be perfectly tracked.  This begs two questions. First will the imperfections in the vehicles controller negate the theoretical cost savings? Second, even if there is still a reduction in cost using our method, does this come at the expense of the tracking error?  That is, are reference trajectories generated using our method harder to track than the Benchmark method?  In this section, we use a basic out-of-the-box controller that has not been optimized to illustrate that even with a sub-optimal controller, there is still substantial cost savings.  Moreover, there is typically even a \emph{reduction} in the tracking error.

In this section. We again consider the methods "Ours" and "Benchmark" defined above. The experiments in this section were carried out using MATLAB's seven degree of freedom driving simulator in its autonomous driving toolbox. The reference trajectories were tracked using MATLAB's Stanley controller with default gains. Here, the maximum error is the maximum lateral error, and the average error is the integral of the lateral error normalized to arc-length of the path. We consider three maneuvers between start goal pairs chosen to coincide with possible way-points of a roundabout. Dimensions and velocities for this problem are in accordance with U.S. department of transportation guidelines \cite[Chapter~6]{robinson2000roundabouts}. The single-lane roundabout, illustrated in Figure \ref{Roundabout}, has lane radius (middle circle) of 45 ft, The initial way-point for this experiment is on a connecting road 65 ft away from the center of the roundabout. We assume an initial velocity of 17 mph. At the second and third way-points, the velocity is assumed to be 14 mph which increases to 17 mph at the final way-point. Trajectories between these way-points were planned for 50 weights: 5 values of the time weight $w_t$, (0, 0.25, 0.5, 0.75, 1) and 10 randomly selected values of the other weights for each value of $w_t$ (ensuring that the sum of the weights equaled 1). Trajectories and their lateral errors for one example weight ($w_{\J}=1, w_t=w_a=w_y=0$) is shown in Figure \ref{MatlabTraj}. The left-most images illustrate the trajectories obtained using the two methods. Blue, red, and green paths represent theoretical, tracked, and driven paths, respectively. Short line segments between red and green paths connect driven configurations with the reference configurations being tracked at each time step. Observe that the high peaks in error (right-most images) experienced by cars following a Benchmark reference trajectory, are spread more evenly along the trajectory for cars tracking Ours trajectories. This is because paths generated using Ours will have more gradual changes in curvature given the example weight. For each value of $w_t$, the values of maximum error and average error of each experiment are tabulated in Figure \ref{MatlabData}, and Table \ref{PathSaveMAT}. Observe that (as implied by the example in Figure \ref{MatlabTraj}), savings in every metric is highest for weights that favor comfort over time. Observe further, that the maximum lateral tracking error can be reduced by approximately 0.8m on average for certain weights. The roundabout radius and velocities were scaled upwards (again following the guidelines of \cite{robinson2000roundabouts}). However, the results were very similar to those reported above.

\vspace{3mm}

\item[Run-time] The focus of this work was on improving the quality of computed trajectories given user weights. As such, we did not focus on optimizing for run time. However, we believe that with proper software engineering, one could substantially reduce the run-time of our procedure though it may still exceed the state-of-the-art (SOTA). 
\begin{table}[htbp]
\centering
\begin{tabular}{@{}c c c c @{}} 
\toprule
& & \multicolumn{2}{c}{Run-times (s)}\\
\cmidrule{3-4}
 Sub-routine & SOTA source & SOTA (C) & Ours (Python)\\
 \midrule
 Calculate a path (Step 1) &  \cite{banzhaf2018g}  & $6.8 \times 10^{-5}$ & $4.3 \times10^{-3}$\\
 Optimize velocity (Step 2) &  \cite{4650924} & $2.2 \times 10 ^{-4}$ & $4.1 \times 10^{-2}$\\
 Total (recursive Step 1 \& 2) & \cite{banzhaf2018g} \& \cite{4650924} & $1.2\times10^{-3}$ & $5.1\times10^{-1}$\\
\bottomrule
\end{tabular}
\caption{Comparison of run-times. Our methods (Ours) compared against state-of-the-art (SOTA) for each sub-routine of the procedure.}
\label{Runtime}
\end{table}
Run-times for each of the two steps in our procedure are compared with those of the SOTA in Table \ref{Runtime} which also includes the total time. Reference \cite{4650924} as the SOTA for velocity profile planning given a path due to its use of third order trapezoidal methods. Trapezoidal methods (and variants thereof) are a widely used method for velocity profile planning \cite{fang2019smooth, 8843215}. Variants of this method use higher-order velocity models than those employed by \cite{4650924}, improving the quality at the expense of computation time. We therefore use computation times reported in \cite{4650924} as a lower bound on trapezoidal velocity method variants. We observe that the SOTA is on the order of 100 times faster than he methods proposed here. However, our algorithm is encoded entirely in Python which tends to run 10-100 times slower than C for iterative procedures like ours.
\end{description}

\section{Concluding Remarks}
We consider the problem of computing trajectories between start and goal states that optimize a trade off between comfort at time.  We offer a simplified approximation of optimization problem \eqref{NEWOP} that replaces the two continuous functions with $n$ constants.  Methods to compute paths given one of these constants, and a velocity profile given this path and the $n-1$ remaining constants are also provided. The resulting technique is verified with numerical examples.

\appendix
\section{Proofs}

\begin{figure}[t]
    \centering
    \includegraphics[width=\linewidth]{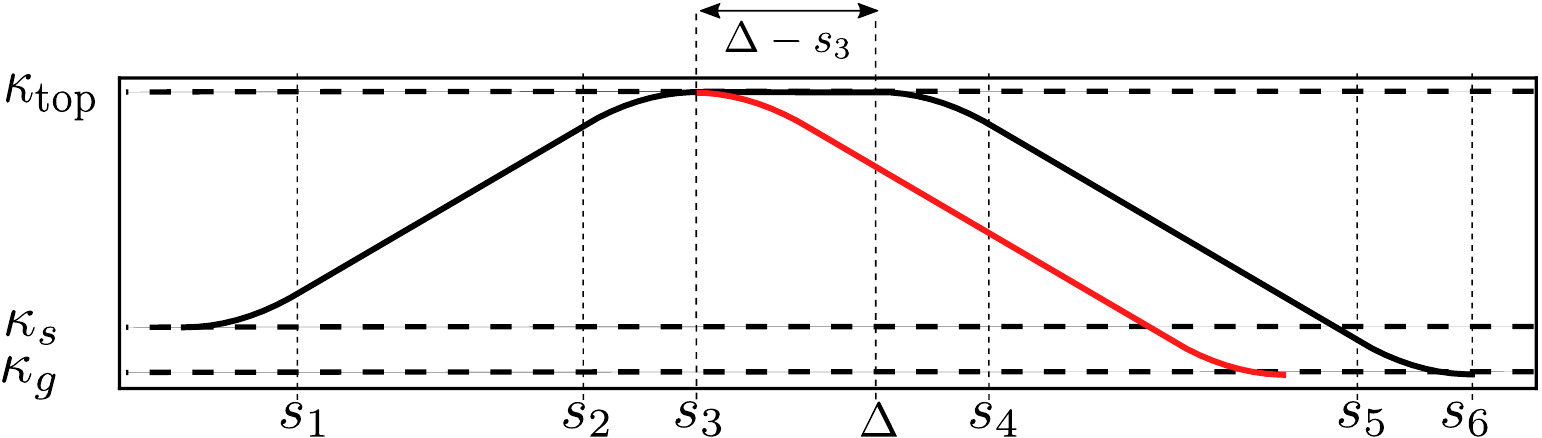}
    \caption{Curvature profile of two curves in the set $\mathcal{G}$. (Red): curvature of the curve $G\varmid$. (Black): curvature of the curve $G\varD$ for $\Delta>s_3$.}  
    \label{IncDelta}
\end{figure}

\begin{proof}[Proof of Lemma \ref{LemmDelta}] Consider a set of curves $\mathcal{G}$.  Let 
$$G(\Delta):\mathbb{R}_{\geq s_3}\rightarrow \mathcal{G},$$
be the function that sends values of $\Delta \geq s_3$ to the corresponding $G3$ curve in $\mathcal{G}$.  Let $x(s, \Delta)$ $y(s, \Delta)$, $\z(s, \Delta)$, $\ka(s, \Delta)$ denote the path state profiles of the curve $G(\Delta)$ at arc-length $s$, and let $x_f(\Delta), y_f(\Delta), \z_f(\Delta), \ka_f(\Delta)$ denote the final path states of $G(\Delta)$.  From Equations~\eqref{dynamics}, \eqref{ks}, we make the following observation:
\small
\be
\label{theta_shift}
\z(s, \Delta)=\begin{cases}
\z(s, s_3), \ &\text{if } s\leq s_3\\
\z(s_3, s_3) + \ktop(s-s_3),& \text{if } s_3<s\leq \Delta\\
\ktop(\Delta - s_3)+\z(s-(\Delta - s_3), s_3),\  &\text{if } \Delta<s\leq s_6.
\end{cases}
\ee
\normalsize

Note (as implied by Figure \ref{IncDelta}), that if $s_6$ is the final arc-length of the curve $G(\Delta)$, then $s_6-(\Delta-s_3)$ is equal to the final arc-length of the curve $G(s_3)$.  Therefore, we can conclude from \eqref{theta_shift} that $\z_f(\Delta) = \ktop (\Delta-s_3)+\z_f(s_3).$ Setting this last to $\z_g$, and solving for $\Delta$ completes the proof.
\end{proof}

\begin{proof}[Proof Of Theorem \ref{curve to circle}] This proof uses the notation of the Proof of Lemma \ref{LemmDelta}.  Given a set of curves $\mathcal{G}$, it can be shown from equations \eqref{theta_shift}, and~\eqref{dynamics}, that
\small
\be
\label{finalsvals}
\begin{split}
 \z_f(\Delta)=&\z_f(s_3)+\ktop(\Delta-s_3),\\
 x_f(\Delta)=&x(s_3, s_3)-\frac{\sin(\z(s_3))-\sin(\ktop(\Delta-s_3)+\z(s_3))}{\ktop}\\
        &+\cos(\ktop(\Delta-s_3))(x_f(s_3)-x(s_3, s_3))\\
        &-\sin(\ktop(\Delta-s_3))(y_f(s_3)-y(s_3, s_3)),\\
        y_f(\Delta)=&y(s_3, s_3)-\frac{-\cos(\z(s_3))+\cos(\ktop(\Delta-s_3)+\z(s_3))}{\ktop}\\
        &+\sin(\ktop(\Delta-s_3))(x_f(s_3)-x(s_3, s_3))\\
        &+\cos(\ktop(\Delta-s_3))(y_f(s_3)-y(s_3, s_3)).
\end{split}
\ee
\normalsize

Observe now that the shortest distance from the point  $\vect{x}_c=[x_c, y_c]^T$ given in \eqref{centers} to the line passing through $(x_f(\Delta), y_f(\Delta))$ whose slope is $\tan{\z_f(\Delta)}$, is given by

$$
d(\Delta)=\frac{ x_c\tan{\z_f(\Delta)} - y_c + y_f(\Delta)-x_f(\Delta)\tan{\z_f(\Delta)}}{\sqrt{\tan{\z_f(\Delta)}^2+1}}.
$$

Replacing \eqref{finalsvals} in $d(\Delta)$, yields a complete cancellation of the term $\Delta$.  The remaining parameters $\ktop, s_3, \vect{x}_c, etc.,$ are identical for all members of $\mathcal{G}$, thus we can conclude that $d(\Delta)$ is the same for all members of $\mathcal{G}$ as desired.
\end{proof}

\bibliographystyle{IEEEtran}
\bibliography{references}

\begin{IEEEbiography}[{\includegraphics[width=1in,height=1.25in,clip,keepaspectratio]{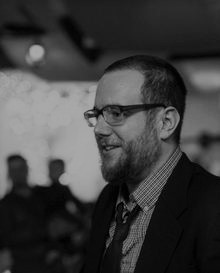}}]{Alexander Botros}
is a Ph.D student in the Autonomous Systems Lab at the University of Waterloo. His research focuses primarily on local planners and trajectory generation for autonomous vehicles. In particular, Alex is researching the problem of computing minimal t-spanning sets of edges for state lattices with the goal of using these sets as motion primitives for autonomous vehicles. Alex Completed his undergraduate and graduate engineering work at Concordia University in Montreal. 
\end{IEEEbiography}

\begin{IEEEbiography}[{\includegraphics[width=1in,height=1.25in,clip,keepaspectratio]{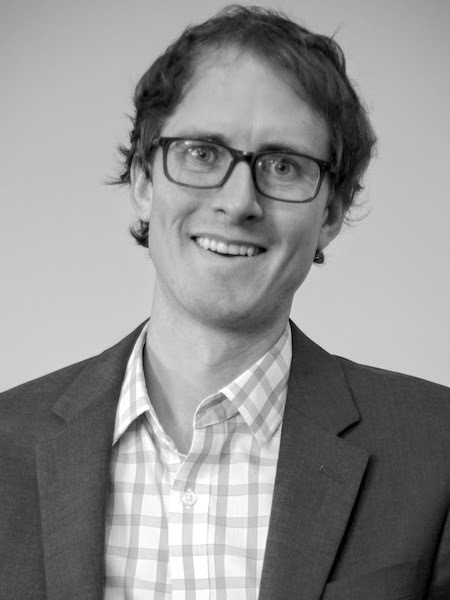}}]{Stephen L. Smith} (S'05--M'09--SM'15) received the B.Sc. degree from Queen’s University, Canada in 2003, the M.A.Sc. degree from the University of Toronto, Canada in 2005, and the Ph.D. degree from UC Santa Barbara, USA in 2009. From 2009 to 2011 he was a Postdoctoral Associate in the Computer Science and Artificial Intelligence Laboratory at MIT, USA.
He is currently an Associate Professor in Electrical and Computer Engineering at the University of Waterloo, Canada and a Canada Research Chair in Autonomous Systems.  His main research interests lie in control and optimization for autonomous systems, with an emphasis on robotic motion planning and coordination.
\end{IEEEbiography}

\end{document}